\documentclass{article} 
\usepackage{arxiv}  
\usepackage{times}  
\usepackage{helvet} 
\usepackage{courier}  
\usepackage[hyphens]{url}  
\usepackage{graphicx} 
\usepackage{natbib}  
\usepackage{caption} 
\usepackage{amsfonts}
\usepackage{nicefrac}
\usepackage{microtype}
\usepackage{bbm}
\usepackage{bm}
\usepackage{amsmath}
\usepackage{amsfonts}
\usepackage{amsthm}
\usepackage{comment}
\usepackage[dvipdfmx]{}
\usepackage{url}

\usepackage{algorithmic}
\usepackage{natbib}
\usepackage{color}
\usepackage{multirow}
\usepackage{booktabs}

\usepackage{braket}

\theoremstyle{definition}

\newtheorem{thm}{Theorem}

\newtheorem{lmm}{Lemma}
\newtheorem{prp}{Proposition}

\newtheorem{remark}{Remark}

\newcommand{\circler}{\raise0.2ex\hbox{\textcircled{\scriptsize{R}}}}

\def \argmin {\mathrm{argmin}}

\title{ATRO: Adversarial Training with a Rejection Option}

\newcommand*{\affaddr}[1]{#1} 
\newcommand*{\affmark}[1][*]{\textsuperscript{#1}}
\newcommand*{\equalcontribution}[1][*]{\textsuperscript{*}}

\author{%
Masahiro Kato\affmark[1],\ \ \ \ \ Zhenghang Cui\affmark[1,2]\footnotemark[1]\thanks{Equal contributions.},\ \ \ \ \ Yoshihiro Fukuhara\affmark[3]\footnotemark[1]\\
\affaddr{\affmark[1]The Univiersity of Tokyo}\\
\affaddr{\affmark[2]RIKEN}\\
\affaddr{\affmark[3]Waseda University}\\
}

\begin{document}


\maketitle

\begin{abstract}
This paper proposes a classification framework with a \emph{rejection option} to mitigate the performance deterioration caused by \emph{adversarial examples}.
While recent machine learning algorithms achieve high prediction performance, they are empirically vulnerable to adversarial examples, which are slightly perturbed data samples that are wrongly classified.
In real-world applications, adversarial attacks using such adversarial examples could cause serious problems.
To this end, various methods are proposed to obtain a classifier that is robust against adversarial examples.
Adversarial training is one of them, which trains a classifier to minimize the worst-case loss under adversarial attacks.
In this paper, in order to acquire a more reliable classifier against adversarial attacks,
we propose the method of \emph{Adversarial Training with a Rejection Option} (ATRO).
Applying the adversarial training objective to both a \emph{classifier} and a \emph{rejection function} simultaneously, classifiers trained by ATRO can choose to abstain from classification when it has insufficient confidence to classify a test data point.
We examine the feasibility of the framework using the surrogate maximum hinge loss and establish a generalization bound for linear models.
Furthermore, we empirically confirmed the effectiveness of ATRO using various models and real-world datasets.
\end{abstract}

\section{Introduction}
\label{sec:introduction}
Recently developed machine learning algorithms show extraordinary performance on classification tasks. However, they can be mislead to make wrong predictions with high confidence by slightly perturbed data.
Adding small perturbations to data to deteriorate the performance of an algorithm is called an \emph{adversarial attack}, and data with such small perturbations are called \emph{adversarial examples} \citep{43405}.
Several previous studies have succeeded in generating adversarial examples by adding small perturbations that are imperceptible to human \citep{43405,Gu2014TowardsDN,huang2015learning,Carlini:2017:AEE:3128572.3140444}. 

Training a robust classifier against such adversarial examples is an important task and one method to increase the classifier robustness is \emph{adversarial training} \citep{42503,43405,journals/corr/ShahamYN15,Carlini,7546524,DBLP:journals/corr/XuEQ17,madry2018towards,buckman2018thermometer,Kannan2018AdversarialLP,NIPS2018_7709,pmlr-v80-wong18a,tramer2018ensemble,zhang2019you,balunovic2020adversarial}.
However, as reported in existing work, most of the defense methods, including adversarial training, fail to avoid misclassification under adversarial attacks.
For example, \citet{Carlini:2017:AEE:3128572.3140444} defeat representative methods for the detection of adversarial examples.
\citet{pmlr-v80-athalye18b} report that several defense algorithms fail in a white-box setting when the attacker uses a carefully designed gradient-based method.
Moreover, \citet{shafahi2018are} and \citet{pmlr-v97-gilmer19a} show that adversarial examples are inevitable in some cases.
Thus, we are playing a cat‐and‐mouse game between algorithms for adversarial attack and defense. 

In order to conduct a highly robust decision making against adversarial examples, this paper incorporates a \emph{rejection option} into the standard adversarial training framework.
The method of \emph{learning with a rejection option} tries to avoid incorrect predictions by abstaining on some uncertain data \citep{Chow1957AnOC,Chow1970OnOR,46544,NIPS2016_6336,NIPS2017_7073,pmlr-v97-geifman19a}.
An approach to learning with a rejection option is to abstain when a sample has a low confidence score.
\citet{46544} generalize this approach and propose a new algorithm that trains both a classifier and a \emph{rejection function}, which decides whether to reject a sample.
\citet{NIPS2017_7073} and \citet{pmlr-v97-geifman19a} formulate learning with a rejection option in another way.
By extending these methods to adversarial training, our method can abstain from decision making when the classifier does not have sufficient confidence, especially on adversarial samples.
We refer to the proposed method as  \emph{adversarial training with a rejection option} (ATRO). 

This paper has two main contributions. First, we show a novel method for robust classification, which incorporates a rejection option into the adversarial training framework.
We further justify it by establishing a generalization error bound.
Next, we investigate the performances of the proposed method on benchmark datasets using both linear-in-parameter models and deep neural networks.

Several previous studies are related to ATRO.
Since the seminal work of \citet{szegedy2013intriguing}, many defense methods have been proposed to mitigate the classifier vulnerability against adversarial examples.
These defense methods include adversarial training \citep{szegedy2013intriguing, goodfellow2015explaining, madry2018towards}, 
the Bayesian method \citep{ye2018bayesian}, detecting adversarial examples \citep{xu2018feature, lu2017safetynet},  pixel defense \citep{song2018pixeldefend}, generative model-based defense  \citep{jalal2017robust, sun2019enhancing}, regularization-based defense \citep{jakubovitz2018improving, qian2019l2-nonexpansive}, tradeoff-based TRADES \citep{zhang2018theoretically}, feature denoising \citep{xie2019feature}, and the method based on the ordinary differential equation viewpoint \citep{zhang2019towards}.
However, most of these defense methods have been shown to be ineffective against sophisticated adversarial attacks which focus on breaking specific defense methods \citep{carlini2017carlini, athalye2018obfuscated}, leaving the problem of defending against adversarial examples still open.

\section{Problem Setting}
\label{sec:problemsetting}
In this paper, we consider the binary classification setting for simplicity.
The proposed method can be extended to the multi-class classification setting, as explained in Remark~\ref{rem:multi-class}--\ref{rem:selective}.

Let $\mathcal{X}\subset\mathbb{R}^d$ denote the input space, where $d$ is the data dimension.
Same as the standard supervised learning setting, we assume both the training and the test data points are i.i.d. samples drawn from an unknown distribution $\mathcal{D}$ over $\mathcal{X}\times\mathcal{Y}$, where $\mathcal{Y}:=\{-1, +1\}$ as we consider the binary classification setting.
Let $\mathcal{F}:=\{f|f:\mathcal{X}\to\mathbb{R}\}$ denote a hypothesis set.
Our goal is to obtain the optimal classifier $f^* := \argmin_{f\in\mathcal{F}}R(f)$, where $R(f):=\mathbb{E}_\mathcal{D}[\mathcal{L}_{0\mathchar`- 1}(f; \bm{x}, y)]$ denotes the population risk and $\mathcal{L}_{0\mathchar`- 1}(f; \bm{x}, y) := \mathbbm{1}_{yf(\bm{x})\leq 0}$ denotes the zero-one loss function.

\subsection{Adversarial Attack}
From the viewpoint of \emph{threat models}, adversarial attacks can be classified into two kinds: \emph{white-box} and \emph{black-box} \citep{yuan2017adversarial}.
In this paper, we only consider white-box attacks, which means that the attacker knows the classification model parameters.
For a given classifier $f:\mathcal{X}\to\mathcal{Z}$ and a data point $(\bm{x}, y)$,
we define the adversarial perturbation as the minimal perturbation $\boldsymbol{\delta}\in\mathbb{R}^d$ which satisfies $\mathrm{sgn}(f(\boldsymbol{x}+\boldsymbol{\delta}))\neq y_i$ and $d(\boldsymbol{x}+\boldsymbol{\delta}, \boldsymbol{x})\leq\varepsilon$ for some $\varepsilon > 0$, where $d$ is a metric.

When using a linear-in-parameter model, we can calculate adversarial examples analytically \citep{pmlr-v97-yin19b}.
However, when using neural networks, obtaining adversarial examples is known to be NP-hard \citep{awasthi2019on}.
Therefore, we usually resort to heuristic methods, such as the \emph{fast gradient sign method} (FGSM) \citep{43405} and the \emph{projected gradient descent method} (PGD) \citep{madry2018towards}, to generate adversarial examples.

FGSM is one of the simplest methods for generating adversarial examples.
It defines a fast single-step attack that maximizes the loss function in the linear approximation.
The perturbation under FGSM is calculated as $\boldsymbol{\delta}_\mathrm{FGSM} = \varepsilon\cdot\mathrm{sgn} ( \nabla_{\boldsymbol{x}} \mathcal{L} ( f; \bm{x},y ) )$, where $\mathcal{L}$ denotes a loss function, e.g., the cross-entropy loss function.

PGD is a standard first-order optimization method, which executes iterative optimizations to find the adversarial perturbation $\boldsymbol{\delta}$.
The perturbation under PGD at a time step $n+1$ is calculated as $\boldsymbol{\delta}^{(n+1)}_{\mathrm{PGD}} = \mathcal{P}_{\varepsilon} ( \boldsymbol{\delta}^{(n)}_{\mathrm{PGD}} + \alpha \cdot \mathrm{sgn} ( \nabla_{\boldsymbol{x}} \mathcal{L} ( f; \bm{x} + \bm{\delta}^{(n)}_{\mathrm{PGD}},y ) ) )$, where $\alpha$ denotes the step size and $\mathcal{P}_{\varepsilon}$ denotes the projection onto the $\ell_p$-ball of radius $\varepsilon$.

\subsection{Adversarial Training}
Adversarial training was first proposed as a method for learning adversarially robust models by injecting adversarial examples into training data \citep{szegedy2013intriguing, goodfellow2015explaining}. 
Following \citet{madry2018towards}, this paper adopts the min-max formulation of adversarial training. Instead of the standard training process, adversarial training defines an $\varepsilon$-ball $\mathbb{B}(\varepsilon)$ around each training data point and solves the following min-max optimization:
\begin{align}
\label{eq:adversarial_training}
\min_{f\in\mathcal{F}}\frac{1}{n}\sum^n_{i=1}\max_{\bm{x}'\in\mathbb{B}^\infty_{\bm{x}}(\varepsilon)}\mathcal{L}(f; \bm{x}', y).
\end{align}

\section{Adversarial Training with a Rejection Option}
\label{sec:atro}
In this section, we introduce a novel framework of adversarial training that has a rejection option.
First, we explain our motivation for incorporating the rejection option.
Then, we propose our algorithm and establish its corresponding generalization error bound.
However, as shown in \cite{pmlr-v97-yin19b}, the exact generalization error bound can be derived only for the linear-in-parameter models.
For neural networks, we need to restrict the model structure to derive the exact generalization error bounds \citep{pmlr-v97-yin19b}.
Therefore, we derive the detailed generalization error bound only for linear-in-parameter models.

\subsection{Motivation of Rejection}
\label{sec:mot_rej}
The adversarial training minimizing the worst-case loss is proposed to construct a robust classifier against adversarial examples.
However, when the model is flexible, it is difficult to train a robust classifier against all possible adversarial examples.
For example, \cite{NIPS2018_7394} proved there exist adversarial perturbations for any classifier, and it is difficult to construct a robust classifier against all adversarial examples.
Therefore, we consider a strategy to abstain from classifying severe adversarial examples.
The idea of classification with abstention has been formulated as learning with a rejection option \citep{Chow1957AnOC,46544,NIPS2017_7073}.
We extend this framework to adversarial training with a rejection option.

\subsection{Learning with a Rejection Option}
\label{sec:ler_ro}
Let $\circler$ denote the rejection symbol.
For any given data point $\bm{x}\in\mathcal{X}$, the classifier has the option of returning the symbol $\circler$, or assigning a label $\hat{y}\in\{-1, +1\}$.
If the classifier rejects on the instance, it incurs a modest cost of $c\in(0, 1/2)$.
If it assigns an incorrect label, it incurs a cost of one. Otherwise, it incurs no loss.
Thus, we formulate the classifier as a pair of functions $(f, r)$, where $f:\mathcal{X}\to\mathbb{R}$ is the function for predicting the label using $\mathrm{sgn}(f)$ and $r:\mathcal{X}\to\mathbb{R}$ is the function for rejecting the data point when $r(\bm{x})\leq 0$.
\citet{46544} defined the loss function as $\mathcal{L}_{0\mathchar`- 1\mathchar`-c}(r, f; \bm{x}, y) = \mathbbm{1}_{y_if(\bm{x})\leq 0} \mathbbm{1}_{r(\bm{x})\geq 0} + c\mathbbm{1}_{r(\bm{x})\leq 0}$ for any pair of functions $(f, r)$ and a labeled sample $(\bm{x}, y)\in\mathcal{X}\times\{-1, +1\}$.
We assume $c\in (0, \frac12)$.
If $c\geq\frac12$, there is no incentive for rejection. For $c=0$, we can reject all data.

Let $\mathcal{F}$ and $\mathcal{R}$ denote two function families mapping $\mathcal{X}$ to $\mathbb{R}$. The learning problem consists of using a labeled sample $S=((\bm{x}_1, y_1),\dots,(\bm{x}_n, y_n))$ drawn i.i.d. from $\mathcal{D}^n$ to determine a pair $(r, f)\in\mathcal{R}\times\mathcal{F}$ with a small expected rejection loss $R(r, f)$ defined as $R(r, f) = \mathbb{E}_{(\bm{x}, y)\sim\mathcal{D}}\left[\mathbbm{1}_{yf(\bm{x})\leq 0} \mathbbm{1}_{r(\bm{x})\geq 0} + c\mathbbm{1}_{r(\bm{x})\leq 0}\right]$.

\subsection{Adversarial Loss with a Rejection Option}
\label{sec:objective-function}
Following the standard formulation of adversarial training of Eq.~\eqref{eq:adversarial_training}, we define the following adversarial loss for $(\bm{x},y)\in S$:
\begin{equation}
\label{eq:adversarial_loss}
\tilde{\mathcal{L}}_{0\mathchar`- 1\mathchar`-c}(r, f; \bm{x}, y):=\max_{\bm{x}'\in\mathbb{B}^\infty_{\bm{x}}(\varepsilon)} \mathcal{L}_{0\mathchar`- 1\mathchar`-c}\left(r, f; \bm{x}', y\right).
\end{equation}
Our goal is to minimize the adversarial population risk $\widetilde{R}_{0\mathchar`- 1\mathchar`-c}(r, f):=\mathbb{E}\left[\tilde{\mathcal{L}}_{0\mathchar`- 1\mathchar`-c}(r, f; \bm{x}, y)\right]$.
To this end, we empirically approximate the above adversarial risk as $\widetilde{R}_{0\mathchar`- 1\mathchar`-c, \mathcal{S}}(r, f):=\frac{1}{n}\sum^n_{i=1}\tilde{\mathcal{L}}_{0\mathchar`- 1\mathchar`-c}(r, f; \bm{x}_i, y_i)$,
and solve $\min_{(r, f) \in \mathcal{R} \times \mathcal{F}} \widetilde{R}_{0\mathchar`- 1\mathchar`-c, \mathcal{S}}(r, f)$.
We refer to this method as \emph{Adversarial Training with a Rejection Option} (ATRO).

\begin{remark}[Separation-based Approach and Confidence-based Approach]
In this paper, we formulate learning with a rejection option following the separation-based approach, which is a generalization of the confidence-based approach. 
In the confidence-based approach, we can only obtain a rejection function with limited function space.
As \citet{46544,NIPS2016_6336} illustrated, a rejection function learned by the separation-based approach can reject samples that a rejection function learned by the confidence-based approach cannot reject. 
\end{remark}

\begin{remark}[Attack only the Classifier]
We can also consider the case where an attacker only attacks the classifier, not the rejection function. In this case, the adversarial attack is weaker than our main setting. This formulation is also practical and can be obtained easily from our formulation of attack to both classifiers and rejection functions.
\end{remark}

\subsection{Convex Surrogate Loss for ATRO}
\label{sec:conv_loss_atro}
Following \cite{46544,NIPS2016_6336}, we consider substituting the zero-one loss with a convex surrogate loss function. 
Let $u\mapsto \Phi(-u)$ and $u\mapsto \Psi(-u)$ be monotonically increasing $L$-Lipschitz and convex functions upper-bounding $\mathbbm{1}_{u\leq 0}$.
Because $\max(a,b)=\frac{a+b+|b-a|}{2}\geq \frac{a+b}{2}$ holds for any $a, b\in\mathbb{R}$, the following inequality holds for $\alpha>0$ and $\beta>0$ \citep{46544}:
\begin{align}
\label{eq:convex_surrogate}
&\mathcal{L}_{0\mathchar`- 1\mathchar`-c}(r, f; \bm{x}, y)\leq \Phi\left(\frac{\alpha}{2}(r(\bm{x})-y_if(\bm{x}))\right) + c\Psi\left(-\beta r(\bm{x})\right).
\end{align}
As $\Phi$ and $\Psi$ are convex functions, their composition with an affine function of $h$ and $r$ is also a convex function of $r$ and $f$.
Thus, the right-hand side of Eq.~\eqref{eq:convex_surrogate} is a convex function of $r$ and $f$.
We consider minimizing the following adversarial loss:
\begin{align}
\label{eq:convex_surrogate_upper}
&\max_{\bm{x}'\in\mathbb{B}^\infty_{\bm{x}}(\varepsilon)}\mathcal{L}_{0\mathchar`- 1\mathchar`-c}(r, f; \bm{x}', y)\nonumber\\
&\leq \max_{\bm{x}'\in\mathbb{B}^\infty_{\bm{x}'}(\varepsilon)} \Big\{\Phi\left(\frac{\alpha}{2}(r(\bm{x}')-yf(\bm{x}'))\right) + c\Psi\left(-\beta r(\bm{x}')\right)\Big\}\nonumber\\
&\leq \max_{\bm{x}'\in\mathbb{B}^\infty_{\bm{x}'}(\varepsilon)}\Big\{\Phi\left(\frac{\alpha}{2}(r(\bm{x}')-yf(\bm{x}')) \right) \Big\}\nonumber\\
&\ \ \ \ \ \ + c\max_{\bm{x}'\in\mathbb{B}^\infty_{\bm{x}'}(\varepsilon)} \Big\{\Psi\left(-\beta r(\bm{x}')\right)\Big\}.
\end{align}

As we assume that $\Phi$ and $\Psi$ are monotonically increasing functions, we can rewrite Eq.~\eqref{eq:convex_surrogate_upper} as
\begin{align*}
&\max_{\bm{x}'\in\mathbb{B}^\infty_{\bm{x}}(\varepsilon)}\left( \Phi\left(\frac{\alpha}{2}(r(\bm{x}')-yf(\bm{x}'))\right) \right)\\
&\ \ \ + c\max_{\bm{x}'\in\mathbb{B}^\infty_{\bm{x}}(\varepsilon)} \left(\Psi\left(-\beta r(\bm{x}')\right)\right)\nonumber\\
&=\Phi\left(\frac{\alpha}{2}\max_{\bm{x}'\in\mathbb{B}^\infty_{\bm{x}}(\varepsilon)}(r(\bm{x}')-yf(\bm{x}'))\right)\\
&\ \ \ + c\Psi\left(\max_{\bm{x}'\in\mathbb{B}^\infty_{\bm{x}}(\varepsilon)}\left(-\beta r(\bm{x}')\right)\right).
\end{align*}
This expression simplifies the algorithm for training ATRO.
Here, we define the adversarial convex surrogate loss as
\begin{align}
\label{eq:convex_surrogate2}
\widetilde{\mathcal{L}}_{\mathrm{conv}}(r, f; \bm{x}, y)=&\Phi\left(\frac{\alpha}{2}\max_{\bm{x}'\in\mathbb{B}^\infty_{\bm{x}}(\varepsilon)}(r(\bm{x}')-yf(\bm{x}'))\right)\nonumber\\
&\ \ \  + c\Psi\left(\max_{\bm{x}'\in\mathbb{B}^\infty_{\bm{x}}(\varepsilon)}\left(-\beta r(\bm{x}')\right)\right),
\end{align}
and the risk with the adversarial convex surrogate loss as $\widetilde{R}_{\mathrm{conv}}(r, f)$.
We denote the sample approximation of $\widetilde{R}_{\mathrm{conv}}(r, f)$ as $\widehat{\widetilde{R}}_{\mathrm{conv}}(r, f)$.

\subsection{Classifiers and Rejection Functions of Linear-in-parameter Models}
By assuming linear-in-parameter models for classifiers and rejection functions, we can calculate the adversarial loss explicitly.
Specifically, assume $r(\bm{x}) = \braket{\bm{x}, \bm{\theta}},\ f(\bm{x}) = \braket{\bm{x}, \bm{\gamma}}$.
By defining $\bm{\zeta}(y)=\bm{\theta}/y-\bm{\gamma}$, we can derive $r(\bm{x}) - yf(\bm{x}) = y\braket{\bm{x}, \bm{\theta}/y-\bm{\gamma}} = y\braket{\bm{x}, \bm{\zeta}(y)}$.
To calculate Eq.~\eqref{eq:convex_surrogate2}, we apply the method of \cite{pmlr-v97-yin19b} and use the following relationships:
\begin{align}
\label{eq:relationship}
&\max_{\bm{x}'\in\mathbb{B}^\infty_{\bm{x}}(\varepsilon)}y\braket{\bm{x}', \bm{\zeta}(y)} = y\braket{\bm{x}, \bm{\zeta}(y)} + \varepsilon\|\bm{\zeta}(y)\|_1,\nonumber\\
&\max_{\bm{x}'\in\mathbb{B}^\infty_{\bm{x}}(\varepsilon)} -\beta r(\bm{x}') = -\beta\left(\braket{\bm{x}, \bm{\gamma}} - \varepsilon\|\bm{\gamma}\|_1\right).
\end{align}
Then, the target risk of ATRO can be written as:
\begin{align}
\label{eq:surrogate_convex}
\widetilde{R}_{\mathrm{conv}}(r, f):=& \mathbb{E}\Big[\Phi\left(\frac{\alpha}{2}\left(y\braket{\bm{x}, \bm{\zeta}(y)} + \varepsilon\|\bm{\zeta}(y)\|_1\right)\right)\\
&\ \ \ + c\Psi\left(-\beta\left(\braket{\bm{x}, \bm{\gamma}(y)} - \varepsilon\|\bm{\gamma}(y)\|_1\right)\right)\Big].\nonumber
\end{align}
The derivation of Eq.~\eqref{eq:relationship} is shown in Appendix~\ref{appdx:deriv_adv}.

\subsection{Generalization Error Bounds for Linear-in-parameter Classifiers and Rejection Functions}
\label{sex:genebound_for_linearclassifier}
Before introducing the algorithm using the surrogate loss Eq.\eqref{eq:surrogate_convex}, we first derive the generalization error bounds between $\widetilde{R}_{\mathrm{conv}}(r, f)$ and $\widehat{\widetilde{R}}_{\mathrm{conv}}(r, f)$, which is the sample approximation of Eq.\eqref{eq:surrogate_convex}.
We use the Rademacher complexity, which is a classic complexity measure used for establishing generalization errors.
For any function class $\mathcal{H}\subseteq \mathbb{R}^\mathcal{Z}$, given a sample $\widetilde{\mathcal{S}}=\{\bm{b}_1, \bm{b}_2, \dots, \bm{b}_n\}$ of size $n$, the empirical Rademacher complexity is defined as 
$\mathfrak{R}_{\widetilde{\mathcal{S}}}(\mathcal{H}):= \frac{1}{n}\mathbb{E}\left[\sup_{h\in\mathcal{H}}\sum^n_{i=1}\sigma_ih(\bm{b}_i)\right]$, where $\sigma_1, \dots, \sigma_n$ are i.i.d.\ Rademacher random variables with $\mathbb{P}\{\sigma_i=1\}=\mathbb{P}\{\sigma_i=-1\}=\frac{1}{2}$.
Let us denote the set of training data points as $\mathcal{S}:=\{(\bm{x}_1, y_1), (\bm{x}_2, y_2), \dots, (\bm{x}_n, y_n)\}$, the function class of the classifier as $\ell_\mathcal{F}=\{\bm{x}\mapsto f(\bm{x}) : f\in\mathcal{F}\}$, and the function class of the rejection function as $\ell_\mathcal{R}=\{\bm{x} \mapsto r(\bm{x}): r\in\mathcal{R}\}$.
Using the Rademacher complexity, we establish the generalization error bounds of the above risk by the following theorem.
\begin{thm}
\label{thm:genbound_conv}
Let $\mathcal{R}$ and $\mathcal{F}$ be families of functions mapping $\mathcal{X}$ to $\mathbb{R}$. Then, for any $\delta > 0$ and $\varepsilon > 0$, with probability at least $1-\delta$ over the draw of a sample $\mathcal{S}$ of size $n$, the following holds for all $(r, f)\in \mathcal{R}\times\mathcal{F}$:
\begin{align*}
&\widetilde{R}_{\mathrm{conv}}(r, f)\\
&\leq \widehat{\widetilde{R}}_{\mathrm{conv}}(r, f) + \frac{\alpha L}{2}\mathfrak{R}_{\mathcal{S}}\left(\left\{\bm{x}\bm{\zeta}(y):\|\bm{\zeta}(y)\|_p\leq W\right\}\right)\\
&+\beta c L\mathfrak{R}_{\mathcal{S}}\left(\left\{\bm{x}\bm{\gamma}:\|\bm{\gamma}\|_p\leq W\right\}\right) + \frac{2\varepsilon Wd^{1/q}}{\sqrt{n}} + \sqrt{\frac{\log\frac{1}{\delta}}{2n}}.
\end{align*}
\end{thm}
The proof is shown in Appendix~\ref{appdx:thm:genbound_conv}.

\subsection{Generalization Error Bounds for Neural Networks}
As \citet{pmlr-v97-yin19b} discussed, it is difficult to derive the generalization error bounds for neural networks because we cannot calculate the exact adversarial example that maximizes the loss.
In order to establish the generalization error bounds, we need to adopt a specific method, which assures a theoretical guarantee, but lacks practical feasibility.
Therefore, we do not establish the generalization error bound for neural networks in our framework because we cannot derive them without model assumptions.

\section{Algorithm with Maximum Hinge Loss}
\citet{46544} used the following maximum hinge (MH) loss for both $u\mapsto\Phi(-u)$ and $u\mapsto\Psi(-u)$, which is an upper bound of Eq.~\eqref{eq:convex_surrogate}:
\begin{align}
\label{eq:mhloss}
&\mathcal{L}_{\mathrm{MH}}(r, f; \bm{x}, y)\\
& = \max\left(1+\frac{\alpha}{2}(r(\bm{x}) - y_ih(\bm{x})), c(1-\beta r(\bm{x})), 0\right).\nonumber
\end{align}
We replace the zero-one loss in Eq.~\eqref{eq:adversarial_loss} with Eq.~\eqref{eq:mhloss}.
Then, we train a classifier and a rejection function by minimizing 
\begin{align*}
\widehat{\widetilde{R}}_{\mathrm{MH}, \mathcal{S}}(r, f)&:= \frac{1}{n}\sum^n_{i=1}\max_{\bm{x}'_i\in\mathbb{B}^\infty_{\bm{x}_i}(\varepsilon)}\mathcal{L}_{\mathrm{MH}}(r, f; \bm{x}_i, y_i)\\
& = \frac{1}{n}\sum^n_{i=1}\widetilde{\mathcal{L}}_{\mathrm{MH}}(r, f; \bm{x}_i, y_i),
\end{align*}
where $\widetilde{\mathcal{L}}_{\mathrm{MH}}(r, f; \bm{x}_i, y_i)=\max_{\bm{x}'_i\in\mathbb{B}^\infty_{\bm{x}_i}(\varepsilon)}\mathcal{L}_{\mathrm{MH}}(r, f; \bm{x}_i, y_i)$.

\begin{remark}
\label{rem:multi-class}
For implementing multi-class classification problem with a rejection option, we need to add some heuristics on the basic formulation of binary classification problem shown in this paper because it is difficult to set a rejection function while maintaining the classification-calibration property \citep{1901.10655}. As a  practical implementation, \citet{NIPS2017_7073} and \citet{pmlr-v97-geifman19a} proposed selective loss and SelectiveNet (SN). We show the details of SN in Remark~\ref{rem:selective}. \citet{kato2020learning} also proposed learning with a rejection option under an adversarial setting using the confidence-based approach.  Thus, based on these existing studies, we can extend the proposed method to multi-class classification. However, because there are various ways to do this, we only show an algorithm for the binary classification that can be the basis of an algorithm for multi-class classification and do not explicitly describe an algorithm for multi-class classification. 
\end{remark}

\subsection{Training Linear-in-parameter Models}
\label{sec:training_with_a_linear}
First, we consider using linear-in-parameter models for $r(\bm{x})$ and $f(\bm{x})$. 
We name the algorithm minimizing the MH loss \emph{ATRO-MH}.
Here, we use the same linear-in-parameter model as shown in Section~\ref{sex:genebound_for_linearclassifier}.
Then, the loss function can be derived as 
\begin{align*}
&\widetilde{\mathcal{L}}_{\mathrm{MH}}(r, f; \bm{x}_i, y_i)\\
&=\max_{\bm{x}'_i\in\mathbb{B}^\infty_{\bm{x}_i}(\varepsilon)}\max\left(1+\frac{\alpha}{2}y_i\bm{x}_i\bm{\zeta}(y_i), c(1-\beta \bm{x}_i\bm{\gamma}), 0\right)\\
&=\max\left(\tilde{A}, \tilde{B}, 0\right),
\end{align*}
where 
\begin{align*}
&\widetilde{A} = 1+\frac{\alpha}{2}\left(y_i\braket{\bm{x}_i, \bm{\zeta}(y_i)} + \varepsilon\|\bm{\zeta}(y_i)\|_1\right),\\
&\widetilde{B} = c\left(1-\beta\left(\braket{\bm{x}_i, \bm{\gamma}} - \varepsilon\|\bm{\gamma}\|_1\right)\right).
\end{align*} 
We use the result shown in Section~\ref{sex:genebound_for_linearclassifier} to derive $\widetilde{A}$ and $\widetilde{B}$.
Therefore, we can train $r(\bm{x})$ and $f(\bm{x})$ by solving the following optimization problem:
\begin{align*}
\min_{\bm{\theta}, \bm{\gamma}, \bm{\xi}} &\ \frac{\lambda}{2}\|\bm{\theta}\|^2 + \frac{\lambda'}{2}\|\bm{\gamma}\|^2 + \sum^n_{i=1}\xi_i\\
\mathrm{s.t.}&\ \xi_i \geq c(1-\beta\left(\braket{\bm{x}_i, \bm{\gamma}} + \varepsilon\|\bm{\gamma}\|_1\right),\\
&\ \xi_i \geq 1+\frac{\alpha}{2}\left(y_i\braket{\bm{x}_i, \bm{\zeta}(y_i)} - \varepsilon\|\bm{\zeta}(y_i)\|_1\right)\nonumber\\
&\ \bm{\zeta}(y_i) = \bm{\theta}/y_i-\bm{\gamma},\ \ \ \xi_i \geq 0,
\end{align*}
where $\lambda$ and $\lambda'$ are regularization parameters.

\begin{table*}[!ht]
    \caption{Results with linear-in-parameter models. We show the error (Err) and rejection rate (Rej) with their average (Mean) and standard deviations (SD). The lowest Err results are in bold.}
    \label{tbl:res_linear_1}
    \begin{center}

    \scalebox{0.83}[0.83]{
    \begin{tabular}{l|l|ll|ll|ll|ll|ll|ll|ll|ll}
\hline
 Dataset &     \multicolumn{17}{c}{australian}\\
 \hline
 Attack &     &   \multicolumn{4}{c|}{Attack $\varepsilon=0$}    &  \multicolumn{4}{c|}{Attack $\varepsilon=0.001$} &   \multicolumn{4}{c|}{Attack $\varepsilon=0.01$} &  \multicolumn{4}{c}{Attack $\varepsilon=0.1$} \\
 \multirow{2}{*}{Method} &  \multirow{2}{*}{Cost} &    Err &        &    Rej &        &    Err &        &    Rej &        &    Err &        &   Rej &       &    Err &      &   Rej &    \\
  &  & Mean &    Std &   Mean &    Std &   Mean &    Std &   Mean &    Std &   Mean &    Std &  Mean &   Std &   Mean &  Std &  Mean &  Std    \\
\hline
SVM & - &  0.199 &  0.014 &  - &  - &  0.316 &  0.022 &  - &  - &  0.527 &  0.023 &  - &  - &  \textbf{0.553} &  0.0 &  - &  - \\
AT & - &  0.195 &  0.005 &  - &  - &  0.192 &  0.005 &  - &  - &  0.237 &  0.033 &  - & - &  \textbf{0.553} &  0.0 &  - & - \\
MH & 0.2 &  0.142 &  0.013 &  0.194 &  0.030 &  0.209 &  0.028 &  0.062 &  0.025 &  0.492 &  0.032 &  0.000 &  0.000 &  \textbf{0.553} &  0.0 &  0.0 &  0.0 \\
      & 0.3 &  0.173 &  0.006 &  0.096 &  0.027 &  0.229 &  0.028 &  0.017 &  0.010 &  0.494 &  0.021 &  0.000 &  0.000 &  \textbf{0.553} &  0.0 &  0.0 &  0.0 \\
      & 0.4 &  0.191 &  0.011 &  0.046 &  0.035 &  0.246 &  0.032 &  0.003 &  0.006 &  0.515 &  0.021 &  0.000 &  0.000 &  \textbf{0.553} &  0.0 &  0.0 &  0.0 \\
ATRO & 0.2 &  \textbf{0.131} &  0.002 &  0.183 &  0.007 &  \textbf{0.130} &  0.002 &  0.179 &  0.009 &  \textbf{0.131} &  0.008 &  0.151 &  0.023 &  \textbf{0.553} &  0.0 &  0.0 &  0.0 \\
      & 0.3 &  0.149 &  0.005 &  0.135 &  0.004 &  0.150 &  0.006 &  0.125 &  0.006 &  0.169 &  0.008 &  0.044 &  0.023 &  \textbf{0.553} &  0.0 &  0.0 &  0.0 \\
      & 0.4 &  0.171 &  0.006 &  0.078 &  0.014 &  0.173 &  0.007 &  0.065 &  0.011 &  0.183 &  0.017 &  0.011 &  0.008 &  \textbf{0.553} &  0.0 &  0.0 &  0.0 \\
\bottomrule
\end{tabular}}
\end{center}

\begin{center}
\scalebox{0.83}[0.83]{
    \begin{tabular}{l|l|ll|ll|ll|ll|ll|ll|ll|ll}
\hline
 Dataset &     \multicolumn{17}{c}{diabetes}\\
 \hline
 Attack &     &   \multicolumn{4}{c|}{Attack $\varepsilon=0$}    &  \multicolumn{4}{c|}{Attack $\varepsilon=0.001$} &   \multicolumn{4}{c|}{Attack $\varepsilon=0.01$} &  \multicolumn{4}{c}{Attack $\varepsilon=0.1$} \\
 \multirow{2}{*}{Method} &  \multirow{2}{*}{Cost} &    Err &        &    Rej &        &    Err &        &    Rej &        &    Err &        &   Rej &       &    Err &      &   Rej &    \\
  &  & Mean &    Std &   Mean &    Std &   Mean &    Std &   Mean &    Std &   Mean &    Std &  Mean &   Std &   Mean &  Std &  Mean &  Std    \\
\hline
SVM & - &  0.292 &  - & - &  0.000 &  0.368 &  0.020 &  - & - &  0.400 &  0.000 &  - & - &  \textbf{0.4} &  0.0 &  - & - \\
AT & - &  0.279 &  0.009 &  - & - &  0.285 &  0.005 &  - & - &  0.307 &  0.006 &  - & - &  \textbf{0.4} &  0.0 &  - & - \\
MH & 0.2 &  \textbf{0.172} &  0.003 &  0.759 &  0.013 &  0.267 &  0.086 &  0.273 &  0.270 &  0.400 &  0.000 &  0.000 &  0.000 &  \textbf{0.4} &  0.0 &  0.0 &  0.0 \\
      & 0.3 &  0.224 &  0.007 &  0.487 &  0.032 &  0.311 &  0.018 &  0.000 &  0.000 &  0.400 &  0.000 &  0.000 &  0.000 &  \textbf{0.4} &  0.0 &  0.0 &  0.0 \\
      & 0.4 &  0.277 &  0.017 &  0.113 &  0.064 &  0.331 &  0.018 &  0.000 &  0.000 &  0.400 &  0.000 &  0.000 &  0.000 &  \textbf{0.4} &  0.0 &  0.0 &  0.0 \\
ATRO & 0.2 &  0.174 &  0.003 &  0.742 &  0.010 &  \textbf{0.173} &  0.002 &  0.737 &  0.007 &  \textbf{0.168} &  0.006 &  0.549 &  0.012 &  \textbf{0.4} &  0.0 &  0.0 &  0.0 \\
      & 0.3 &  0.228 &  0.006 &  0.497 &  0.011 &  0.231 &  0.004 &  0.490 &  0.008 &  0.253 &  0.008 &  0.355 &  0.015 &  \textbf{0.4} &  0.0 &  0.0 &  0.0 \\
      & 0.4 &  0.279 &  0.009 &  0.000 &  0.000 &  0.283 &  0.006 &  0.000 &  0.000 &  0.308 &  0.006 &  0.000 &  0.000 &  \textbf{0.4} &  0.0 &  0.0 &  0.0 \\
\hline
\end{tabular}}
\end{center}

\begin{center}
\scalebox{0.83}[0.83]{
    \begin{tabular}{l|l|ll|ll|ll|ll|ll|ll|ll|ll}
\hline
 Dataset &     \multicolumn{17}{c}{cod-rna}\\
 \hline
 Attack &     &   \multicolumn{4}{c|}{Attack $\varepsilon=0$}    &  \multicolumn{4}{c|}{Attack $\varepsilon=0.001$} &   \multicolumn{4}{c|}{Attack $\varepsilon=0.01$} &  \multicolumn{4}{c}{Attack $\varepsilon=0.1$} \\
 \multirow{2}{*}{Method} &  \multirow{2}{*}{Cost} &    Err &        &    Rej &        &    Err &        &    Rej &        &    Err &        &   Rej &       &    Err &      &   Rej &    \\
  &  & Mean &    Std &   Mean &    Std &   Mean &    Std &   Mean &    Std &   Mean &    Std &  Mean &   Std &   Mean &  Std &  Mean &  Std    \\
\hline
SVM & - &  0.109 &  0.133 &  - & - &  0.197 &  0.277 &  - & - &  0.289 &  0.354 &  - & - &  \textbf{0.288} &  0.353 &  - & - \\
AT & - &  0.113 &  0.139 &  - & - &  0.113 &  0.139 &  - & - &  0.155 &  0.224 &  - & - &  \textbf{0.288} &  0.353 &  - & - \\
MH & 0.2 &  \textbf{0.079} &  0.097 &  0.395 &  0.483 &  \textbf{0.079} &  0.097 &  0.391 &  0.479 &  0.288 &  0.353 &  0.000 &  0.000 &  \textbf{0.288} &  0.353 &  0.0 &  0.0 \\
      & 0.3 &  0.115 &  0.141 &  0.312 &  0.402 &  0.215 &  0.279 &  0.099 &  0.292 &  0.279 &  0.342 &  0.000 &  0.000 &  \textbf{0.288} &  0.353 &  0.0 &  0.0 \\
      & 0.4 &  0.112 &  0.138 &  0.002 &  0.006 &  0.227 &  0.298 &  0.000 &  0.000 &  0.288 &  0.353 &  0.000 &  0.000 &  \textbf{0.288} &  0.353 &  0.0 &  0.0 \\
ATRO & 0.2 &  \textbf{0.079} &  0.097 &  0.397 &  0.486 &  \textbf{0.079} &  0.097 &  0.395 &  0.483 &  0.288 &  0.353 &  0.000 &  0.000 &  \textbf{0.288} &  0.353 &  0.0 &  0.0 \\
      & 0.3 &  0.112 &  0.137 &  0.091 &  0.223 &  0.112 &  0.137 &  0.090 &  0.219 &  0.153 &  0.212 &  0.013 &  0.027 &  \textbf{0.288} &  0.353 &  0.0 &  0.0 \\
      & 0.4 &  0.115 &  0.141 &  0.007 &  0.016 &  0.116 &  0.142 &  0.007 &  0.016 &  \textbf{0.115} &  0.142 &  0.005 &  0.011 &  \textbf{0.288} &  0.353 &  0.0 &  0.0 \\
\bottomrule
\end{tabular}}
\end{center}

\begin{center}
\scalebox{0.83}[0.83]{
    \begin{tabular}{l|l|ll|ll|ll|ll|ll|ll|ll|ll}
\hline
 Dataset &     \multicolumn{17}{c}{skin}\\
 \hline
 Attack &     &   \multicolumn{4}{c|}{Attack $\varepsilon=0$}    &  \multicolumn{4}{c|}{Attack $\varepsilon=0.001$} &   \multicolumn{4}{c|}{Attack $\varepsilon=0.01$} &  \multicolumn{4}{c}{Attack $\varepsilon=0.1$} \\
 \multirow{2}{*}{Method} &  \multirow{2}{*}{Cost} &    Err &        &    Rej &        &    Err &        &    Rej &        &    Err &        &   Rej &       &    Err &      &   Rej &    \\
  &  & Mean &    Std &   Mean &    Std &   Mean &    Std &   Mean &    Std &   Mean &    Std &  Mean &   Std &   Mean &  Std &  Mean &  Std    \\
\hline
SVM & - &  0.006 &  0.005 &  - & - &  0.009 &  0.010 &  - & - &  0.245 &  0.369 &  - & - &  \textbf{0.700} &  0.146 &  - & - \\
AT & - &  0.004 &  0.003 &  - & - &  0.005 &  0.003 &  - & - &  0.090 &  0.115 &  - & - &  0.779 &  0.066 &  - & - \\
MH & 0.2 &  0.008 &  0.004 &  0.000 &  0.000 &  0.010 &  0.006 &  0.000 &  0.000 &  0.216 &  0.327 &  0.0 &  0.0 &  0.747 &  0.116 &  0.0 &  0.0 \\
      & 0.3 &  0.007 &  0.004 &  0.001 &  0.002 &  0.009 &  0.007 &  0.000 &  0.000 &  0.319 &  0.388 &  0.0 &  0.0 &  0.751 &  0.116 &  0.0 &  0.0 \\
      & 0.4 &  0.008 &  0.005 &  0.001 &  0.002 &  0.011 &  0.012 &  0.000 &  0.000 &  0.234 &  0.351 &  0.0 &  0.0 &  0.727 &  0.134 &  0.0 &  0.0 \\
ATRO & 0.2 &  0.007 &  0.003 &  0.001 &  0.002 &  0.007 &  0.000 &  0.001 &  0.002 &  \textbf{0.042} &  0.075 &  0.0 &  0.0 &  0.749 &  0.094 &  0.0 &  0.0 \\
      & 0.3 &  0.004 &  0.003 &  0.000 &  0.000 &  0.004 &  0.003 &  0.000 &  0.000 &  0.059 &  0.117 &  0.0 &  0.0 &  0.720 &  0.106 &  0.0 &  0.0 \\
      & 0.4 &  \textbf{0.002} &  0.003 &  0.000 &  0.000 &  \textbf{0.003} &  0.003 &  0.000 &  0.000 &  0.059 &  0.114 &  0.0 &  0.0 &  0.792 &  0.045 &  0.0 &  0.0 \\
\hline
\end{tabular}}
\end{center}

\end{table*}

\subsection{Training Neural Networks}
In this section, we consider using neural networks for modeling $r(\bm{x})$ and $f(\bm{x})$.
The deep linear SVM~\citep{tang2013deep} is a model structure of deep neural networks with a linear SVM as a top layer instead of a softmax layer. 
Adopting this structure, we can train the model by minimizing $\max(1-y_if(\boldsymbol{x}_i),0)^{2} 
+ \frac{\lambda}{2}||\boldsymbol{w}||^2_2$, where $\boldsymbol{w}$ is the weight of the final layer of $f(\bm{x})$ and $\lambda$ is a regularization parameter. 
To adapt the ATRO framework, we substitute the squared hinge loss with the MH loss (\ref{eq:mhloss}). Then, we train the models by minimizing the following adversarial loss:
\begin{align}
\label{eq:atro_with_neural_networks}
\max_{\bm{x}'_i\in\mathbb{B}^\infty_{\bm{x}_i}(\varepsilon)}
\Bigg\{
\max \Big(& 1+\frac{\alpha}{2}(r(\bm{x}_i) - y_{i}f(\bm{x}_{i}),\nonumber\\
& c(1-\beta r(\bm{x}_{i})), 0 \Big)^2+ \frac{\lambda}{2}||\bm{w}||^2_2
\Bigg\}
,
\end{align}
where $\alpha$, $\beta$ and $c$ are the same as in Section \ref{sec:conv_loss_atro}. 
Note that we cannot solve $\max_{\bm{x}'_i\in\mathbb{B}^\infty_{\bm{x}_i}(\varepsilon)}\{\cdot\}$ efficiently in many cases because the optimization is known to be NP hard \citep{awasthi2019on}.
Therefore, instead of solving the maximization problem directly, we solve the inner optimization problem by using existing heuristic algorithms such as FGSM \citep{goodfellow2015explaining} and PGD \citep{madry2018towards}. 

In experiments, we set $\lambda=1$ and adopted the VGG-16-based architecture \citep{simonyan2014}. The differences from the original VGG-16 architecture are two-fold: (i) there is only one fully connected layer with $512$ neurons and (ii) batch normalization \citep{ioffe2015batch} and dropout \citep{srivastava2014dropout} are applied.

\begin{remark}[SelectiveNet]
\label{rem:selective}
As introduced in Remark~\ref{rem:multi-class}, as a different formulation, we can also implement ATRO using the SelectiveNet (SN) \citep{pmlr-v97-geifman19a}, which is a different method for learning with a rejection option. SN is a method for end-to-end training of both a classifier and rejection function using neural networks. Although there is no detailed theoretical guarantee unlike MH based method, SN is reported to show a preferable performance extended to multi-class classification problem. In the neural network structure of the SN, the input $\bm{x}$ is firstly processed by the body block $h$, which can be assembled using any types of architecture. Then, $h(\bm{x})$ is fed into three outputs, a classifier $f_{h}$, rejection function $r_{h}$, and auxiliary prediction $g_{h}$. Auxiliary prediction $g_{h}$ is only used for training to prevent overfitting and enforce the construction of relevant features in the body block $h$. Let us introduce a brief formulation as follows. For the details about SN, see the original paper \citep{pmlr-v97-geifman19a}. For a given target coverage $0 < c \leq 1$, the target risk is defined as $R_{\mathrm{sel}}(r_{h},f_{h}) = \frac{\frac{1}{n} \sum_{i=1}^{n} \mathcal{L}_{\mathrm{ce}}(f_{h}; \boldsymbol{x}_i, y_i)r_{h}(\boldsymbol{x}_i)}{\hat{\phi}(r_{h};\boldsymbol{x}_i)}+
\lambda \max(0, c-\hat{\phi}(r_{h};\boldsymbol{x}_i))$, where $\mathcal{L}_{\mathrm{ce}}$ is the standard cross-entropy loss and $\hat{\phi}(r_{h})=\frac{1}{n}\sum_{i=1}^{n}r_{h}(\boldsymbol{x}_i)$ is empirical coverage and $\lambda$ is a regularization parameter.
Then, the total objective is defined as $R(r_{h},f_{h},g_{h}) = 
\eta R_{\mathrm{sel}}(r_{h},f_{h}) +
(1-\eta)R_{\mathrm{ce}}(g_{h})$, where $0\leq\eta\leq 1$ is a convex combination weight and $R_{\mathrm{ce}}$ is the empirical risk of cross-entropy. Next, we consider adding adversarial training to the original SN. By applying adversarial training, we train a classifier and rejection function as the minimizer of
\begin{align*}
&\max_{\bm{x}'_i\in\mathbb{B}^\infty_{\bm{x}_i}(\varepsilon)} \Bigg\{ \eta\Bigg(
\frac{\frac{1}{n} \sum_{i=1}^{n} \mathcal{L}_{\mathrm{ce}}(f_{h}; \boldsymbol{x}_i, y_i)r_{h}(\boldsymbol{x}_i)}{\hat{\phi}(r_{h};\boldsymbol{x}_i)}\\
& +
\lambda \max(0, c-\hat{\phi}(r_{h};\boldsymbol{x}_i)) \Bigg) + (1-\eta)  \mathcal{L}_{\mathrm{ce}}(g_{h}; \boldsymbol{x}_i, y_i) \Bigg\}.
\end{align*}
\end{remark}

\begin{table*}[!ht]
    \caption{Results of deep linear SVM with a rejection option (SVM) and deep linear SVM with ATRO (ATRO). We show the error (Err), rejection rate (Rej), and precision of rejection (PR).}
    \label{result-deepsvm}
    \begin{center}
    
    \begin{tabular}{c|c|c c c |c c c} \hline
    \multicolumn{2}{c|}{} & \multicolumn{3}{c|}{CIFAR-10} & \multicolumn{3}{c}{SVHN} \\
    \hline
    \multirow{2}{*}{Training} & \multirow{2}{*}{Adversarial Examples} & \multicolumn{3}{c|}{Metric} & \multicolumn{3}{c}{Metric} \\ 
    & & \hfil Err \hfil  & \hfil Rej \hfil & \hfil PR \hfil & \hfil Err \hfil  & \hfil Rej \hfil & \hfil PR \hfil \\
    \hline
    \multirow{3}{*}{SVM}  & no attack  & 3.80$\pm$0.16  & 1.10$\pm$0.09  & 28.4$\pm$4.21 & 1.00$\pm$0.13  & 0.20$\pm$0.03  & 12.1$\pm$2.91  \\
    & $\ell_{\infty}$ bounded by $4/255$ & 26.7$\pm$1.59  & 4.20$\pm$0.46  & 33.1$\pm$3.89 & 11.7$\pm$1.68  & 1.50$\pm$0.43  & 38.5$\pm$4.25  \\
    & $\ell_{\infty}$ bounded by $8/255$ & 32.7$\pm$3.79  & 4.10$\pm$0.79  & 32.5$\pm$5.44 & 14.6$\pm$3.45  & 1.60$\pm$0.63  & 38.6$\pm$4.21  \\
    \hline
    \multirow{6}{*}{ATRO} &  no attack & 4.00$\pm$0.26  & 1.60$\pm$0.11  & 46.9$\pm$4.85 & 0.90$\pm$0.03  & 0.40$\pm$0.04  & 21.9$\pm$3.78  \\
    & $\ell_{\infty}$ bounded by $4/255$ & 11.3$\pm$0.40  & 2.60$\pm$0.40  & 48.9$\pm$3.63 & 4.10$\pm$0.06  & 1.40$\pm$0.28  & 42.2$\pm$1.03  \\
    & $\ell_{\infty}$ bounded by $8/255$ & 16.7$\pm$0.32  & 3.80$\pm$0.32  & 48.3$\pm$5.84 & 7.50$\pm$0.41  & 2.20$\pm$0.51  & 44.2$\pm$4.34  \\
    & $\ell_{2}$ bounded by $80$  & 8.90$\pm$0.75  & 2.60$\pm$0.39  & 47.4$\pm$3.14 & 3.50$\pm$0.11  & 1.20$\pm$0.25  & 40.2$\pm$3.47  \\
    & $\ell_{2}$ bounded by $160$  & 12.4$\pm$0.61  & 2.70$\pm$0.31  & 44.3$\pm$2.41 & 6.70$\pm$0.38  & 2.00$\pm$0.39  & 43.8$\pm$3.17  \\
    \hline
    \end{tabular}
    \end{center}
\end{table*}

\begin{table*}[t]
    \caption{Results of SN with a rejection option and SN with ATRO. We show the error (Err), rejection rate (Rej), and precision of rejection (PR).}
    \label{result-selectivenet}
    \begin{center}
    
    \begin{tabular}{c|c|c c c |c c c} \hline
    \multicolumn{2}{c|}{} & \multicolumn{3}{c|}{CIFAR-10} & \multicolumn{3}{c}{SVHN} \\
    \hline
    \multirow{2}{*}{Training} & \multirow{2}{*}{Adversarial Examples} & \multicolumn{3}{c|}{Metric} & \multicolumn{3}{c}{Metric} \\ 
    & & \hfil Err \hfil  & \hfil Rej \hfil & \hfil PR \hfil & \hfil Err \hfil  & \hfil Rej \hfil & \hfil PR \hfil \\
    \hline
    \multirow{3}{*}{SN} & no attack & 3.30$\pm$0.03  & 22.8$\pm$0.44  & 63.7$\pm$2.02 & 1.91$\pm$0.04  & 12.5$\pm$0.47  & 73.4$\pm$1.09  \\
    & $\ell_{\infty}$ bounded by $4/255$ & 55.5$\pm$0.53  & 42.3$\pm$0.97  & 24.4$\pm$0.30 & 25.8$\pm$0.72  & 44.2$\pm$1.57  & 48.5$\pm$0.46  \\
    & $\ell_{\infty}$ bounded by $8/255$ & 67.7$\pm$0.63  & 49.3$\pm$1.03  & 18.2$\pm$0.41 & 34.5$\pm$1.57  & 49.3$\pm$1.49  & 37.2$\pm$0.81  \\
    \hline
    \multirow{5}{*}{ATRO} & no attack  & 8.04$\pm$0.43  & 20.8$\pm$1.01  & 54.9$\pm$1.07 & 1.50$\pm$0.41  & 19.8$\pm$4.88  & 75.4$\pm$3.22  \\
    & $\ell_{\infty}$ bounded by $4/255$ & 29.5$\pm$0.88  & 26.4$\pm$1.35  & 37.2$\pm$1.11 & 7.65$\pm$4.52  & 30.7$\pm$2.51  & 60.1$\pm$3.17  \\
    & $\ell_{\infty}$ bounded by $8/255$ & 46.8$\pm$1.12  & 28.3$\pm$1.43  & 24.5$\pm$0.51 & 13.3$\pm$9.06  & 34.9$\pm$2.95  & 52.4$\pm$2.05  \\
    & $\ell_{2}$ bounded by $80$ & 20.3$\pm$0.78  & 24.9$\pm$1.26  & 43.6$\pm$0.71 & 8.01$\pm$2.85  & 33.0$\pm$4.22  & 58.0$\pm$8.01  \\
    & $\ell_{2}$ bounded by $160$ & 33.6$\pm$0.88  & 26.9$\pm$1.56  & 34.9$\pm$1.23 & 14.4$\pm$5.45  & 37.2$\pm$5.04  & 47.6$\pm$1.14  \\
    \hline
    \end{tabular}
    \end{center}
\end{table*}

\section{Experiments}
\label{sec:exp}
In this section, we report the experimental results of ATRO.

\subsection{Benchmark Test using Linear Models}
First, we investigated the performance of ATRO with a linear-in-parameter model. Following \citet{46544}, we used the australian, pima, skin, and cod datasets\footnote{\url{https://www.csie.ntu.edu.tw/~cjlin/libsvmtools/datasets/}}. For a linear-in-parameter model, we used a Gaussian kernel. We compared the proposed method ATRO with the standard SVM, learning with rejection under MH loss, and the standard adversarial training with SVM. We used $\varepsilon=0, 0.001, 0.01, 0.1$ to train a classifier with AT (adversarial training without rejection)  and ATRO. The adversarial attack $\delta$ was chosen from $0$, $0.001$, $0.01$, and $0.1$. The costs of MH and ATRO were chosen from $0.2$, $0.3$, and $0.4$. For each experiment, we trained a classifier with $10$ trials and $500$ samples. For the error (Err) and the percentage of rejected samples over all $500$ samples (Rej), we calculated their means (Mean) and standard deviations (Std). Let true accept (TA) be an outcome where the rejection function correctly accepts data such that a sample $\bm{x}$ does not belong to the defense target class, i.e., $h(\bm{x})\neq t$. Let true reject (TR) be an outcome where the rejection function correctly rejects data such that a sample $\bm{x}$ belongs to the defense target class, i.e., $h(\bm{x})=t$. FA and FR can be defined similarly \citep{NIPS2019_8527}. Then, the rejection rate is defined as $\frac{\mathrm{TR}+\mathrm{FR}}{\mathrm{TR}+\mathrm{TA}+\mathrm{FR}+\mathrm{FA}}$. In Table~\ref{tbl:res_linear_1}, we show the results with australian and pima datasets under adversarial attacks such that $|\bm{\delta}|_\infty = 0.0001, 0.001, 0.01, 0.1$ while using an adversarial training parameter $\varepsilon=0.001$. Other results are shown in Appendix~\ref{appdx:det_exp}. 

In most cases, the ATRO shows the better performance by appropriately abstaining from the classification. Moreover, the ATRO shows robustness against the adversarial attack $\bm{\delta}=0.01$ larger than the parameter in training $\varepsilon=0.001$. We show additional experimental results using $\varepsilon=0.1, 0.01$ for adversarial training and without adversarial training in Appendix~\ref{appdx:det_exp}.

\subsection{Benchmark Test using Neural Networks}
We compare the ATRO using a deep linear SVM with a plain deep linear SVM with a rejection option without adversarial training. We train the classifier using the loss described in \eqref{eq:atro_with_neural_networks}. As benchmark data, we use the CIFAR-10\footnote{See \url{https://www.cs.toronto.edu/~kriz/cifar.html}.}  \citep{Krizhevsky09learningmultiple} and SVHN\footnote{See \url{http://ufldl.stanford.edu/housenumbers/}.} \citep{svhn} datasets. The CIFAR-10 dataset is an image classification dataset comprising a training set of $50,000$ images and $10,000$ test images classified into $10$ categories. The image size is $32 \times 32 \times 3$ pixels (RGB images). The SVHN dataset \citep{svhn} is an image classification dataset containing $73,257$ training images and $26,032 $test images classified into $10$ classes representing digits. The images are digits of house street numbers, which were cropped and aligned after being taken from the Google Street View service. The image size is $32 \times 32 \times 3$ pixels (RGB images). To make a benchmark binary classification problem, we select the airplane class for CIFAR-10 and character zero class for SVHN as true classes. We set the cost as $c=0.3$. The results are shown in Table~\ref{result-deepsvm}. For conducting adversarial training and constructing adversarial examples, we use the PGD method with $\ell_{\infty}$-bounded attack to generate training samples. Test adversarial examples are generated by the PGD method with $\ell_{\infty}$ and $\ell_{2}$-bounded attack. For training the classier, as the bound of $\ell_{\infty}$-bounded attack, we set the distortion size as $\varepsilon=4/255$ and applied random uniform scaling $(0, \varepsilon)$ to improve performance against smaller distortions. Then we select an attack target class for each image uniformly at random from the set of incorrect classes. We used $20$ optimization steps and a step size of $\varepsilon / \sqrt{\mathrm{steps}}$, as described in \citet{1908.08016}. For generating test adversarial examples, as the bound of $\ell_{\infty}$-bounded attack, we use $4/255$ and $8/255$; as the bound of $\ell_{2}$-bounded attack, we use $80$ and $160$. At test time, we used $200$-step targeted attacks with a uniform random (incorrect) target class, using the best practice of employing more steps for evaluation than for training \citep{carlini2019on}. For each experiment, we train a classifier with $3$ trials and report the error (Err), the percentage of rejected samples over all samples (Rej), and precision of rejection (PR) \citep{NIPS2019_8527}. Using $\mathrm{TR}$ and $\mathrm{FR}$, the precision of rejection is defined as $\frac{\mathrm{TR}}{\mathrm{TR}+\mathrm{FR}}$ \citep{NIPS2019_8527}. We performed training on a single NVIDIA V100 GPU using standard data augmentation \citep{he2016identity}. The network was optimized using stochastic gradient descent with a momentum of $0.9$, an initial learning rate of $0.1$, and a weight decay of $5\mathrm{e}\mathchar`-4$. The learning rate was reduced by $0.5$ every $25$ epochs, and training continued for $200$ epochs. 

The experimental results suggest the effectiveness of neural network extensions of ATRO. In particular, ATRO with a deep linear SVM successfully reduces the error on the adversarial examples by about half, without deterioration of the error on the standard samples. 
We also note that the decreased error for different types of attack used at the training stage indicates that extended ATRO is not strongly dependent on the method used to generate adversarial examples.

\begin{remark}[Experimental Results with SN]
We also report the result of ATRO with SN. In the experiments, we adopted the VGG-16-based architecture as the body block of SN, which is same approach taken in the original paper \citep{pmlr-v97-geifman19a}, and we set the values $\eta=0.5$ and $\lambda=32$. The experimental results are shown in Table~\ref{result-selectivenet}.
\end{remark}

\section{Conclusion}
The existence of adversarial examples is a notorious problem in machine learning. In this paper, we propose an idea of allowing a classifier to reject suspicious samples and a method called adversarial training with a rejection option (ATRO). The motivation is to avoid classification decision making when we do not have sufficient confidence to determine the class under adversarial attacks. We describe the implementation of ATRO for both linear-in-parameter models and neural network models. In experiments, the method successfully decreased the error by rejecting uncertain samples.

\clearpage

\section*{Ethics Statement}
The vulnerability against adversarial attacks is a serious issue for many machine learning algorithms from the security aspect. Especially, the misclassification caused by adversarial attack leads to serious consequences in some real-world applications, e.g., autonomous driving. The theoretical analysis and proposed method in this paper have the potential to alleviate this concern. Specifically, our approach allows to reduce serious misclassification by rejecting all possible adversarial examples. Therefore, this paper has a broad impact on improving the robustness of machine learning applications. On the other hand, while the proposed method improves adversarial robustness, it does not ensure to reject all attacks. In addition, a proposed method also has the potential risk of encouraging unfair decision making by abstaining prediction about a particular group of people.

\bibliography{atro.bbl}
\bibliographystyle{icml2020}

\clearpage
\onecolumn

\appendix

\setcounter{secnumdepth}{3}

\section{Derivation of Equation (\ref{eq:relationship})}
\label{appdx:deriv_adv}
We follow the derivation by \citet{pmlr-v97-yin19b}. After defining $\tilde g_{\bm{\zeta}(y_i)}(\bm{x}_i, y_i) := \max_{\bm{x}'_i\in\mathbb{B}^\infty_{\bm{x}_i}(\varepsilon)}y_i\braket{\bm{x}_i, \bm{\zeta}(y_i)}$, we have
\begin{align*}
\tilde g_{\bm{\zeta}(y_i)}(\bm{x}_i, y_i) = \begin{cases}
\max_{\bm{x}'_i\in\mathbb{B}^\infty_{\bm{x}_i}(\varepsilon)}\braket{\bm{x}_i, \bm{\zeta}(y_i)} & y= + 1,\\
-\min_{\bm{x}'_i\in\mathbb{B}^\infty_{\bm{x}_i}(\varepsilon)}\braket{\bm{x}_i, \bm{\zeta}(y_i)} & y = -1.\\
\end{cases}
\end{align*}
When $y=+1$, we have
\begin{align*}
\tilde g_{\bm{\zeta}(y_i)}(\bm{x}_i, y_i)  &= \tilde g_{\bm{\zeta}(y_i)}(\bm{x}_i, 1) =  \max_{\bm{x}'_i\in\mathbb{B}^\infty_{\bm{x}_i}(\varepsilon)}\braket{\bm{x}_i, \bm{\zeta}(y_i)}\\
& = \sum^D_{d=1} \big[\mathbbm{1}(\bm{\zeta}(y_i)_d \geq 0)(\bm{x}_d + \varepsilon) + \mathbbm{1}(\bm{\zeta}(y_i)_d < 0)(\bm{x}_d - \varepsilon)\big]\\
& = \sum^D_{d=1}\bm{\zeta}(y_i)_d\big(\bm{x}_d + \mathrm{sgn}(\bm{\zeta}(y_i)_d)\varepsilon\big)\\
& = \braket{\bm{x}_i, \bm{\zeta}(y_i)} + \varepsilon\big\|\bm{\zeta}(y_i)\big\|_1.
\end{align*}
Similarly, when $y=-1$, we have
\begin{align*}
\tilde g_{\bm{\zeta}(y_i)}(\bm{x}_i, y_i)  &= \tilde g_{\bm{\zeta}(y_i)}(\bm{x}_i, -1) =  - \min_{\bm{x}'_i\in\mathbb{B}^\infty_{\bm{x}_i}(\varepsilon)}\braket{\bm{x}_i, \bm{\zeta}(y_i)}\\
& = - \sum^D_{d=1} \big[\mathbbm{1}(\bm{\zeta}(y_i)_d \geq 0)(\bm{x}_d - \varepsilon) + \mathbbm{1}(\bm{\zeta}(y_i)_d < 0)(\bm{x}_d + \varepsilon)\big]\\
& = - \sum^D_{d=1}\bm{\zeta}(y_i)_d\big(\bm{x}_d - \mathrm{sgn}(\bm{\zeta}(y_i)_d)\varepsilon\big)\\
& = - \braket{\bm{x}_i, \bm{\zeta}(y_i)} + \varepsilon\big\|\bm{\zeta}(y_i)\big\|_1.
\end{align*}
Thus, we conclude that
\begin{align*} 
\tilde g_{\bm{\zeta}(y_i)}(\bm{x}_i, y_i) := \max_{\bm{x}'_i\in\mathbb{B}^\infty_{\bm{x}_i}(\varepsilon)}y_i\braket{\bm{x}_i, \bm{\zeta}(y_i)}.
\end{align*}

\section{Proof of Theorem~\ref{thm:genbound_conv}}
\label{appdx:thm:genbound_conv}
To derive generalization bounds for the risk with the adversarial convex loss, we need to know the upper bound of the Rademacher complexity of $\tilde{\mathcal{F}}:=\{\min_{\bm{x}'_i\in\mathbb{B}^\infty_{\bm{x}_i}(\varepsilon)} y_i\braket{\bm{x}'_i, \bm{w}}: \|\bm{w}\|_p\leq W\}$, where $\bm{w}$ is the parameter of the linear-in-parameter model.
\citet{pmlr-v97-yin19b} derived the following result for the upper bound of the Rademacher complexity.
\begin{prp}[\citet{pmlr-v97-yin19b}, Theorem~2.]
\label{prp:advrad}
Let $\mathcal{F}:=\{f_{\bm{w}}(\bm{x}_i):\|\bm{w}\|_p\leq W\}$ and $\tilde{\mathcal{F}}:=\{\min_{\bm{x}'_i\in\mathbb{B}^\infty_{\bm{x}_i}(\varepsilon)} y_i\braket{\bm{x}'_i, \bm{w}}: \|\bm{w}\|_p\leq W\}$. Suppose that $\frac{1}{p}+\frac{1}{q}=1$. Then, there exists a universal constant $c\in(0,1)$ such that
\begin{align*}
&\max\left\{\mathfrak{R}_{\mathcal{S}}(\mathcal{F}), c\varepsilon W\frac{d^{\frac{1}{q}}}{\sqrt{n}}\right\}\leq \mathfrak{R}_{\mathcal{S}}(\tilde{\mathcal{F}})\leq \mathfrak{R}_{\mathcal{S}}(\mathcal{F}) + \varepsilon W \frac{d^{\frac{1}{q}}}{\sqrt{n}}.
\end{align*}
\end{prp}
We apply the same technique used by \citet{pmlr-v97-yin19b} to bound $\mathfrak{R}_{\mathcal{S}}\left(\left\{\bm{x}_i\bm{\zeta}(y_i):\|\bm{\zeta}(y_i)\|_p\leq W\right\}\right)$ and $\mathfrak{R}_{\mathcal{S}}\left(\left\{\bm{x}_i\bm{\gamma}:\|\bm{\gamma}\|_p\leq W\right\}\right)$.

\begin{lmm}
\label{lmm:adv_rad_atro}
Suppose that $\frac{1}{p}+\frac{1}{q}=1$. Then,
\begin{align*}
&\mathfrak{R}_{\mathcal{S}}\left(\left\{\max_{\bm{x}'_i\in\mathbb{B}^\infty_{\bm{x}_i}(\varepsilon)}y_i\bm{x}_i\bm{\zeta}(y_i):\|\bm{\zeta}(y_i)\|_p\leq W\right\}\right)\leq \mathfrak{R}_{\mathcal{S}}\left(\left\{\bm{x}_i\bm{\zeta}(y_i):\|\bm{\zeta}(y_i)\|_p\leq W\right\}\right) + \varepsilon W \frac{d^{1/q}}{\sqrt{n}},\\
&\mathfrak{R}_{\mathcal{S}}\left(\left\{\min_{\bm{x}'_i\in\mathbb{B}^\infty_{\bm{x}_i}(\varepsilon)}\bm{x}_i\bm{\gamma}:\|\bm{\gamma}\|_p\leq W\right\}\right)\leq \mathfrak{R}_{\mathcal{S}}\left(\left\{\bm{x}_i\bm{\gamma}:\|\bm{u}\|_p\leq W\right\}\right) + \frac{\varepsilon W d^{1/q}}{\sqrt{n}}.
\end{align*}
\end{lmm}
\begin{proof}
First, we derive the first inequality. We have
\begin{align}
\label{eq:rad1}
&\mathfrak{R}_{\mathcal{S}}\left(\left\{\max_{\bm{x}'_i\in\mathbb{B}^\infty_{\bm{x}_i}(\varepsilon)}y_i\bm{x}_i\bm{\zeta}(y_i):\|\bm{\zeta}(y_i)\|_p\leq W\right\}\right)\nonumber\\
&=\mathfrak{R}_{\mathcal{S}}\left(\left\{y_i\braket{\bm{x}_i, \bm{\zeta}(y_i)} - \varepsilon\|\bm{\zeta}\|_1:\|\bm{\zeta}(y_i)\|_p\leq W\right\}\right)\nonumber\\
&=\frac{1}{n}\mathbb{E}_{\sigma}\left[\sup_{\|\bm{\zeta}(y_i)\|_p\leq W}\sum^n_{i=1}\sigma_i\left(y_i\braket{\bm{x}_i, \bm{\zeta}(y_i)} - \varepsilon\|\bm{\zeta}\|_1\right)\right].
\end{align}
Then, using the same technique in the proof of Proposition~\ref{prp:advrad}, we bound (\ref{eq:rad1}) as follows:
\begin{align*}
&\mathfrak{R}_{\mathcal{S}}\left(\left\{\max_{\bm{x}'_i\in\mathbb{B}^\infty_{\bm{x}_i}(\varepsilon)}y_i\bm{x}_i\bm{\zeta}(y_i):\|\bm{\zeta}(y_i)\|_p\leq W\right\}\right)\leq \mathfrak{R}_{\mathcal{S}}\left(\left\{\bm{x}_i\bm{\zeta}(y_i):\|\bm{\zeta}(y_i)\|_p\leq W\right\}\right) + \varepsilon W \frac{d^{1/q}}{\sqrt{n}}.
\end{align*}
Similarly, we bound $\mathfrak{R}_{\mathcal{S}}\left(\left\{\min_{\bm{x}'_i\in\mathbb{B}^\infty_{\bm{x}_i}(\varepsilon)}\bm{x}_i\bm{\gamma}:\|\bm{\gamma}\|_p\leq W\right\}\right)$ as follows:
\begin{align*}
&\mathfrak{R}_{\mathcal{S}}\left(\left\{\min_{\bm{x}'_i\in\mathbb{B}^\infty_{\bm{x}_i}(\varepsilon)}\bm{x}_i\bm{\gamma}:\|\bm{\gamma}\|_p\leq W\right\}\right)\\
&=\frac{1}{n}\mathbb{E}_{\sigma}\left[\sup_{\|\bm{\gamma}\|_p\leq W}\sum^n_{i=1}\sigma_i\left(\braket{\bm{x}_i, \bm{\gamma}}\right) - \varepsilon\sum^n_{i=1}\sigma_i\|\bm{\gamma}\|_1\right]\\
&=\frac{1}{n}\mathbb{E}_{\sigma}\left[\sup_{\|\bm{\gamma}\|_p\leq W}\braket{\bm{u}, \bm{\gamma}} - \nu\braket{\bm{\gamma}, \mathrm{sgn}(\bm{u}})\right]\\
&=\frac{1}{n}\mathbb{E}_{\sigma}\left[\sup_{\|\bm{\gamma}\|_p\leq W}\braket{\bm{\gamma}, \bm{u}  - \nu\mathrm{sgn}(\bm{u}})\right]\\
&\leq \frac{W}{n}\mathbb{E}_{\sigma}\left[\left\|\sum^n_{i=1}\sigma_i\bm{x}_i - \varepsilon\sum^n_{i=1}\sigma_i\mathrm{sgn}\left(\sum^n_{i=1}\sigma_i\bm{x}_i\right)\right\|_q\right]\\
&\leq \mathfrak{R}_{\mathcal{S}}\left(\left\{\bm{x}_i\bm{\gamma}:\|\bm{\gamma}\|_p\leq W\right\}\right) + \frac{W}{n}\mathbb{E}_{\sigma}\left[\left|\varepsilon\sum^n_{i=1}\sigma_i\mathrm{sgn}\left(\sum^n_{i=1}\sigma_i\bm{x}_i\right)\right|\right]\\
&\leq \mathfrak{R}_{\mathcal{S}}\left(\left\{\bm{x}_i\bm{\gamma}:\|\bm{u}\|_p\leq W\right\}\right) + \frac{\varepsilon W d^{1/q}}{\sqrt{n}},
\end{align*}
where $\nu:=\varepsilon\sum^n_{i=1}\sigma_i$. 
\end{proof}
Next, we show the proof of Theorem~\ref{thm:genbound_conv}.
\begin{proof}[Proof of Theorem~\ref{thm:genbound_conv}]
Let $\widetilde{\ell}_{\mathrm{conv}, \mathcal{H}, \mathcal{R}}$ be the family of functions defined by 
\begin{align*}
&\widetilde{\ell}_{\mathrm{conv}, \mathcal{H}, \mathcal{R}} = \left\{(\bm{x}_i, y_i)\mapsto\min\left(\widetilde{\mathcal{L}}_{\mathrm{conv}}(r, f; \bm{x}_i, y_i), 1\right), r\in\mathcal{R}, f\times\mathcal{F}\right\}.
\end{align*}
Then, the Rademacher complexity of $\widetilde{\ell}_{\mathrm{conv}, \mathcal{H}, \mathcal{R}}$ can be written as follows:
\begin{align*}
&\mathfrak{R}(\widetilde{\ell}_{\mathrm{conv}, \mathcal{H}, \mathcal{R}})\\
& = \mathbb{E}\left[\sup_{h,r\in\mathcal{H}\times\mathcal{R}}\frac{1}{n}\sum^n_{i=1}\sigma_i\widetilde{\mathcal{L}}_{\mathrm{conv}, \mathcal{H}, \mathcal{R}}(r, f; \bm{x}_i, y_i)\right]\\
&\leq \mathbb{E}\left[\sup_{h,r\in\mathcal{H}\times\mathcal{R}}\frac{1}{n}\sum^n_{i=1}\sigma_i\Phi\left(\max_{\bm{x}'_i\in\mathbb{B}^\infty_{\bm{x}_i}(\varepsilon)}\left(\frac{\alpha}{2}y_i\bm{x}_i\bm{\zeta}(y_i)\right)\right)\right]+ \mathbb{E}\left[\sup_{h,r\in\mathcal{H}\times\mathcal{R}}\frac{1}{n}\sum^n_{i=1}\sigma_i\Psi\left(\max_{\bm{x}'_i\in\mathbb{B}^\infty_{\bm{x}_i}(\varepsilon)}\left(-c\beta \bm{x}_i\bm{\gamma})\right)\right)\right].
\end{align*}
Because $\Phi$ and $\Psi$ are $L$-Lipschitz functions, by Talagrand's contraction lemma \citep{LedouxTalagrand91book},
\begin{align*}
&\mathfrak{R}_{\mathcal{S}}(\widetilde{\ell}_{\mathrm{MH}, \mathcal{H}, \mathcal{R}})\\
&\leq \frac{\alpha L}{2}\mathfrak{R}_{\mathcal{S}}\left(\left\{\max_{\bm{x}'_i\in\mathbb{B}^\infty_{\bm{x}_i}(\varepsilon)}y_i\bm{x}_i\bm{\zeta}(y_i):\|\bm{\zeta}(y_i)\|_p\leq W\right\}\right) +\beta c L\mathfrak{R}_{\mathcal{S}}\left(\left\{\min_{\bm{x}'_i\in\mathbb{B}^\infty_{\bm{x}_i}(\varepsilon)}\bm{x}_i\bm{\gamma}:\|\bm{\gamma}\|_p\leq W\right\}\right).
\end{align*}
From Lemma~\ref{lmm:adv_rad_atro}, we have the result of Theorem~\ref{thm:genbound_conv}.
\end{proof}

\section{Additional Experiments with Linear-in-parameter Models}
\label{appdx:det_exp}
We also investigated the performance of ATRO with linear-in-parameter models.
In Section~\ref{sec:exp}, we show the results with an adversarial training parameter $\varepsilon=0.001$.
In this section, we additionally show the results with an adversarial training parameter $\varepsilon=0$ in Table~\ref{tbl:res_linear_2}, $\varepsilon=0.01$ in Table~\ref{tbl:res_linear_3}, and $\varepsilon=0.1$ in Table~\ref{tbl:res_linear_4}, respectively.
Other settings are identical to that of Section~\ref{sec:exp}.

In the case when the adversarial training parameter $\varepsilon=0$, there is no difference between the SVM and AT; there is also no difference between the MH and ATRO.
The results of Table~\ref{tbl:res_linear_2} reflects the fact.
As increasing the adversarial training parameter value, although ATRO becomes robust to adversarial attacks, the performance against the lower adversarial attacks decreased.
This tendency is common in adversarial training.
However, ATRO showed better performance than ATRO in many cases.

\begin{table*}[!ht]
    \caption{Results with linear-in-parameter models with an adversarial training parameter $\varepsilon=0$. We show the error (Err) and rejection rate (Rej) with their average (Mean) and standard deviations (SD). The lowest Err results are in bold.}
    \label{tbl:res_linear_2}
    \begin{center}

    \scalebox{0.83}[0.83]{
    \begin{tabular}{l|l|ll|ll|ll|ll|ll|ll|ll|ll}
\hline
 Dataset &     \multicolumn{17}{c}{australian}\\
 \hline
 Attack &     &   \multicolumn{4}{c|}{Attack $\varepsilon=0$}    &  \multicolumn{4}{c|}{Attack $\varepsilon=0.001$} &   \multicolumn{4}{c|}{Attack $\varepsilon=0.01$} &  \multicolumn{4}{c}{Attack $\varepsilon=0.1$} \\
 \multirow{2}{*}{Method} &  \multirow{2}{*}{Cost} &    Err &        &    Rej &        &    Err &        &    Rej &        &    Err &        &   Rej &       &    Err &      &   Rej &    \\
  &  & Mean &    Std &   Mean &    Std &   Mean &    Std &   Mean &    Std &   Mean &    Std &  Mean &   Std &   Mean &  Std &  Mean &  Std    \\
\hline
SVM & - &  0.204 &  0.017 & - & - &  0.335 &  0.025 &  - & - &  0.526 &  0.026 &  - & - &  \textbf{0.553} &  0.0 &  - & - \\
AT & - &  0.204 &  0.017 &  - & - &  0.335 &  0.025 &  - & - &  0.526 &  0.026 &  - & - &  \textbf{0.553} &  0.0 &  - & - \\
MH & 0.2 &  0.147 &  0.014 &  0.187 &  0.020 &  \textbf{0.199} &  0.018 &  0.057 &  0.023 &  0.497 &  0.030 &  0.0 &  0.0 &  \textbf{0.553} &  0.0 &  0.0 &  0.0 \\
      & 0.3 &  0.177 &  0.013 &  0.103 &  0.023 &  0.230 &  0.031 &  0.007 &  0.010 &  0.509 &  0.033 &  0.0 &  0.0 &  \textbf{0.553} &  0.0 &  0.0 &  0.0 \\
      & 0.4 &  0.185 &  0.006 &  0.060 &  0.024 &  0.241 &  0.020 &  0.002 &  0.003 &  0.511 &  0.025 &  0.0 &  0.0 &  \textbf{0.553} &  0.0 &  0.0 &  0.0 \\
ATRO & 0.2 &  \textbf{0.145} &  0.014 &  0.197 &  0.031 &  \textbf{0.199} &  0.024 &  0.058 &  0.023 &  \textbf{0.492} &  0.025 &  0.0 &  0.0 &  \textbf{0.553} &  0.0 &  0.0 &  0.0 \\
      & 0.3 &  0.170 &  0.008 &  0.086 &  0.017 &  0.232 &  0.024 &  0.017 &  0.014 &  0.505 &  0.013 &  0.0 &  0.0 &  \textbf{0.553} &  0.0 &  0.0 &  0.0 \\
      & 0.4 &  0.190 &  0.010 &  0.055 &  0.031 &  0.259 &  0.024 &  0.005 &  0.010 &  0.511 &  0.020 &  0.0 &  0.0 &  \textbf{0.553} &  0.0 &  0.0 &  0.0 \\
\bottomrule
\end{tabular}}
\end{center}

\begin{center}
\scalebox{0.83}[0.83]{
    \begin{tabular}{l|l|ll|ll|ll|ll|ll|ll|ll|ll}
\hline
 Dataset &     \multicolumn{17}{c}{diabetes}\\
 \hline
 Attack &     &   \multicolumn{4}{c|}{Attack $\varepsilon=0$}    &  \multicolumn{4}{c|}{Attack $\varepsilon=0.001$} &   \multicolumn{4}{c|}{Attack $\varepsilon=0.01$} &  \multicolumn{4}{c}{Attack $\varepsilon=0.1$} \\
 \multirow{2}{*}{Method} &  \multirow{2}{*}{Cost} &    Err &        &    Rej &        &    Err &        &    Rej &        &    Err &        &   Rej &       &    Err &      &   Rej &    \\
  &  & Mean &    Std &   Mean &    Std &   Mean &    Std &   Mean &    Std &   Mean &    Std &  Mean &   Std &   Mean &  Std &  Mean &  Std    \\
\hline
SVM & - &  0.292 &  0.011 &  - & - &  0.367 &  0.017 &  - & - &  \textbf{0.4} &  0.0 &  - & - &  \textbf{0.4} &  0.0 &  - & - \\
AT & - &  0.292 &  0.011 &  - & - &  0.367 &  0.017 &  - & - &  \textbf{0.4} &  0.0 &  - & - &  \textbf{0.4} &  0.0 &  - & - \\
MH & 0.2 &  0.173 &  0.004 &  0.751 &  0.012 &  \textbf{0.238} &  0.082 &  0.370 &  0.250 &  \textbf{0.4} &  0.0 &  0.0 &  0.0 &  \textbf{0.4} &  0.0 &  0.0 &  0.0 \\
      & 0.3 &  0.221 &  0.006 &  0.490 &  0.027 &  0.307 &  0.025 &  0.000 &  0.000 &  \textbf{0.4} &  0.0 &  0.0 &  0.0 &  \textbf{0.4} &  0.0 &  0.0 &  0.0 \\
      & 0.4 &  0.281 &  0.017 &  0.087 &  0.074 &  0.329 &  0.017 &  0.000 &  0.000 &  \textbf{0.4} &  0.0 &  0.0 &  0.0 &  \textbf{0.4} &  0.0 &  0.0 &  0.0 \\
ATRO & 0.2 &  \textbf{0.171} &  0.001 &  0.751 &  0.011 &  0.278 &  0.074 &  0.215 &  0.257 &  \textbf{0.4} &  0.0 &  0.0 &  0.0 &  \textbf{0.4} &  0.0 &  0.0 &  0.0 \\
      & 0.3 &  0.221 &  0.006 &  0.465 &  0.043 &  0.311 &  0.015 &  0.000 &  0.000 &  \textbf{0.4} &  0.0 &  0.0 &  0.0 &  \textbf{0.4} &  0.0 &  0.0 &  0.0 \\
      & 0.4 &  0.272 &  0.013 &  0.121 &  0.070 &  0.325 &  0.016 &  0.000 &  0.000 &  \textbf{0.4} &  0.0 &  0.0 &  0.0 &  \textbf{0.4} &  0.0 &  0.0 &  0.0 \\
\hline
\end{tabular}}
\end{center}

\begin{center}
\scalebox{0.83}[0.83]{
    \begin{tabular}{l|l|ll|ll|ll|ll|ll|ll|ll|ll}
\hline
 Dataset &     \multicolumn{17}{c}{cod-rna}\\
 \hline
 Attack &     &   \multicolumn{4}{c|}{Attack $\varepsilon=0$}    &  \multicolumn{4}{c|}{Attack $\varepsilon=0.001$} &   \multicolumn{4}{c|}{Attack $\varepsilon=0.01$} &  \multicolumn{4}{c}{Attack $\varepsilon=0.1$} \\
 \multirow{2}{*}{Method} &  \multirow{2}{*}{Cost} &    Err &        &    Rej &        &    Err &        &    Rej &        &    Err &        &   Rej &       &    Err &      &   Rej &    \\
  &  & Mean &    Std &   Mean &    Std &   Mean &    Std &   Mean &    Std &   Mean &    Std &  Mean &   Std &   Mean &  Std &  Mean &  Std    \\
\hline
SVM & - &  0.113 &  0.139 &  - & - &  0.115 &  0.141 &  - & - &  \textbf{0.243} &  0.320 &  - & - &  \textbf{0.288} &  0.353 &  - & - \\
AT & - &  0.113 &  0.139 &  - & - &  0.115 &  0.141 &  - & - &  \textbf{0.243} &  0.320 &  - & - &  \textbf{0.288} &  0.353 &  - & - \\
MH & 0.2 &  \textbf{0.079} &  0.097 &  0.397 &  0.486 &  \textbf{0.078} &  0.096 &  0.392 &  0.480 &  0.288 &  0.353 &  0.000 &  0.000 &  \textbf{0.288} &  0.353 &  0.0 &  0.0 \\
      & 0.3 &  0.119 &  0.146 &  0.357 &  0.451 &  0.149 &  0.201 &  0.294 &  0.448 &  0.281 &  0.345 &  0.000 &  0.000 &  \textbf{0.288} &  0.353 &  0.0 &  0.0 \\
      & 0.4 &  0.113 &  0.138 &  0.003 &  0.010 &  0.153 &  0.215 &  0.003 &  0.010 &  0.287 &  0.351 &  0.001 &  0.004 &  \textbf{0.288} &  0.353 &  0.0 &  0.0 \\
ATRO & 0.2 &  \textbf{0.079} &  0.097 &  0.393 &  0.481 &  0.079 &  0.097 &  0.389 &  0.477 &  0.287 &  0.352 &  0.000 &  0.000 &  \textbf{0.288} &  0.353 &  0.0 &  0.0 \\
      & 0.3 &  0.111 &  0.137 &  0.218 &  0.347 &  0.172 &  0.231 &  0.097 &  0.290 &  0.280 &  0.343 &  0.000 &  0.000 &  \textbf{0.288} &  0.353 &  0.0 &  0.0 \\
      & 0.4 &  0.104 &  0.128 &  0.003 &  0.008 &  0.268 &  0.330 &  0.003 &  0.008 &  0.288 &  0.353 &  0.000 &  0.000 &  \textbf{0.288} &  0.353 &  0.0 &  0.0 \\

\bottomrule
\end{tabular}}
\end{center}

\begin{center}
\scalebox{0.83}[0.83]{
    \begin{tabular}{l|l|ll|ll|ll|ll|ll|ll|ll|ll}
\hline
 Dataset &     \multicolumn{17}{c}{skin}\\
 \hline
 Attack &     &   \multicolumn{4}{c|}{Attack $\varepsilon=0$}    &  \multicolumn{4}{c|}{Attack $\varepsilon=0.001$} &   \multicolumn{4}{c|}{Attack $\varepsilon=0.01$} &  \multicolumn{4}{c}{Attack $\varepsilon=0.1$} \\
 \multirow{2}{*}{Method} &  \multirow{2}{*}{Cost} &    Err &        &    Rej &        &    Err &        &    Rej &        &    Err &        &   Rej &       &    Err &      &   Rej &    \\
  &  & Mean &    Std &   Mean &    Std &   Mean &    Std &   Mean &    Std &   Mean &    Std &  Mean &   Std &   Mean &  Std &  Mean &  Std    \\
\hline
SVM & - &  0.009 &  0.004 &  - & - &  0.010 &  0.004 &  - & - &  0.203 &  0.308 &  - & - &  0.761 &  0.073 &  - & - \\
AT & - &  0.009 &  0.004 &  - &  - &  0.010 &  0.004 &  - &  - &  0.203 &  0.308 &  - &  - &  0.761 &  0.073 &  - &  - \\
MH & 0.2 &  0.007 &  0.006 &  0.0 &  0.0 &  0.011 &  0.007 &  0.0 &  0.0 &  0.387 &  0.371 &  0.0 &  0.0 &  0.748 &  0.113 &  0.0 &  0.0 \\
      & 0.3 &  \textbf{0.005} &  0.004 &  0.0 &  0.0 &  0.006 &  0.005 &  0.0 &  0.0 &  0.201 &  0.288 &  0.0 &  0.0 &  \textbf{0.666} &  0.138 &  0.0 &  0.0 \\
      & 0.4 &  \textbf{0.005} &  0.004 &  0.0 &  0.0 &  0.008 &  0.006 &  0.0 &  0.0 &  0.321 &  0.326 &  0.0 &  0.0 &  0.740 &  0.149 &  0.0 &  0.0 \\
ATRO & 0.2 &  0.006 &  0.002 &  0.0 &  0.0 &  0.007 &  0.005 &  0.0 &  0.0 &  0.161 &  0.297 &  0.0 &  0.0 &  0.749 &  0.094 &  0.0 &  0.0 \\
      & 0.3 &  0.007 &  0.003 &  0.0 &  0.0 &  0.012 &  0.009 &  0.0 &  0.0 &  0.237 &  0.357 &  0.0 &  0.0 &  0.737 &  0.113 &  0.0 &  0.0 \\
      & 0.4 &  0.006 &  0.005 &  0.0 &  0.0 &  \textbf{0.005} &  0.004 &  0.0 &  0.0 &  \textbf{0.006} &  0.004 &  0.0 &  0.0 &  0.707 &  0.113 &  0.0 &  0.0 \\
\hline
\end{tabular}}
\end{center}

\end{table*}

\begin{table*}[!ht]
    \caption{Results with linear-in-parameter models with an adversarial training parameter $\varepsilon=0.01$. We show the error (Err) and rejection rate (Rej) with their average (Mean) and standard deviations (SD). The lowest Err results are in bold. }
    \label{tbl:res_linear_3}
    \begin{center}

    \scalebox{0.83}[0.83]{
    \begin{tabular}{l|l|ll|ll|ll|ll|ll|ll|ll|ll}
\hline
 Dataset &     \multicolumn{17}{c}{australian}\\
 \hline
 Attack &     &   \multicolumn{4}{c|}{Attack $\varepsilon=0$}    &  \multicolumn{4}{c|}{Attack $\varepsilon=0.001$} &   \multicolumn{4}{c|}{Attack $\varepsilon=0.01$} &  \multicolumn{4}{c}{Attack $\varepsilon=0.1$} \\
 \multirow{2}{*}{Method} &  \multirow{2}{*}{Cost} &    Err &        &    Rej &        &    Err &        &    Rej &        &    Err &        &   Rej &       &    Err &      &   Rej &    \\
  &  & Mean &    Std &   Mean &    Std &   Mean &    Std &   Mean &    Std &   Mean &    Std &  Mean &   Std &   Mean &  Std &  Mean &  Std    \\
\hline
SVM & - &  0.197 &  0.010 &  - & - &  0.311 &  0.024 &  - & - &  0.535 &  0.018 &  - & - &  0.553 &  0.000 &  - & - \\
AT & - &  0.180 &  0.000 & - & - &  0.180 &  0.000 &  - & - &  0.180 &  0.000 &  - & - &  0.517 &  0.019 &  - & - \\
MH & 0.2 &  0.147 &  0.010 &  0.192 &  0.026 &  0.202 &  0.016 &  0.058 &  0.021 &  0.508 &  0.024 &  0.000 &  0.000 &  0.553 &  0.000 &  0.0 &  0.0 \\
      & 0.3 &  0.168 &  0.009 &  0.107 &  0.024 &  0.225 &  0.029 &  0.015 &  0.009 &  0.499 &  0.029 &  0.000 &  0.000 &  0.553 &  0.000 &  0.0 &  0.0 \\
      & 0.4 &  0.186 &  0.012 &  0.047 &  0.021 &  0.237 &  0.023 &  0.002 &  0.003 &  0.509 &  0.028 &  0.000 &  0.000 &  0.553 &  0.000 &  0.0 &  0.0 \\
ATRO & 0.2 &  \textbf{0.128} &  0.000 &  0.173 &  0.000 &  \textbf{0.128} &  0.000 &  0.173 &  0.000 &  \textbf{0.128} &  0.000 &  0.173 &  0.000 &  0.542 &  0.007 &  0.0 &  0.0 \\
      & 0.3 &  0.145 &  0.000 &  0.173 &  0.000 &  0.145 &  0.000 &  0.173 &  0.000 &  0.145 &  0.000 &  0.173 &  0.000 &  \textbf{0.437} &  0.051 &  0.0 &  0.0 \\
      & 0.4 &  0.175 &  0.004 &  0.032 &  0.024 &  0.176 &  0.004 &  0.031 &  0.024 &  0.178 &  0.003 &  0.017 &  0.016 &  0.464 &  0.028 &  0.0 &  0.0 \\
\bottomrule
\end{tabular}}
\end{center}

\begin{center}
\scalebox{0.83}[0.83]{
    \begin{tabular}{l|l|ll|ll|ll|ll|ll|ll|ll|ll}
\hline
 Dataset &     \multicolumn{17}{c}{diabetes}\\
 \hline
 Attack &     &   \multicolumn{4}{c|}{Attack $\varepsilon=0$}    &  \multicolumn{4}{c|}{Attack $\varepsilon=0.001$} &   \multicolumn{4}{c|}{Attack $\varepsilon=0.01$} &  \multicolumn{4}{c}{Attack $\varepsilon=0.1$} \\
 \multirow{2}{*}{Method} &  \multirow{2}{*}{Cost} &    Err &        &    Rej &        &    Err &        &    Rej &        &    Err &        &   Rej &       &    Err &      &   Rej &    \\
  &  & Mean &    Std &   Mean &    Std &   Mean &    Std &   Mean &    Std &   Mean &    Std &  Mean &   Std &   Mean &  Std &  Mean &  Std    \\
\hline
SVM & - &  0.298 &  0.013 &  - & - &  0.369 &  0.022 &  - & - &  0.400 &  0.000 &  - & - &  0.400 &  0.000 &  - & - \\
AT & - &  0.309 &  0.015 &  - & - &  0.312 &  0.017 &  - & - &  0.324 &  0.014 &  - & - &  0.399 &  0.002 &  - & - \\
MH & 0.2 &  \textbf{0.171} &  0.002 &  0.751 &  0.012 &  0.285 &  0.069 &  0.195 &  0.238 &  0.400 &  0.000 &  0.000 &  0.000 &  0.400 &  0.000 &  0.000 &  0.000 \\
      & 0.3 &  0.223 &  0.007 &  0.485 &  0.038 &  0.313 &  0.023 &  0.000 &  0.000 &  0.400 &  0.000 &  0.000 &  0.000 &  0.400 &  0.000 &  0.000 &  0.000 \\
      & 0.4 &  0.269 &  0.020 &  0.127 &  0.072 &  0.331 &  0.015 &  0.000 &  0.000 &  0.400 &  0.000 &  0.000 &  0.000 &  0.400 &  0.000 &  0.000 &  0.000 \\
ATRO & 0.2 &  0.191 &  0.006 &  0.953 &  0.032 &  \textbf{0.191} &  0.006 &  0.953 &  0.032 &  \textbf{0.191} &  0.006 &  0.950 &  0.033 &  \textbf{0.190} &  0.007 &  0.907 &  0.060 \\
      & 0.3 &  0.244 &  0.003 &  0.686 &  0.016 &  0.242 &  0.004 &  0.673 &  0.022 &  0.234 &  0.006 &  0.563 &  0.019 &  0.376 &  0.013 &  0.037 &  0.024 \\
      & 0.4 &  0.337 &  0.007 &  0.793 &  0.022 &  0.335 &  0.007 &  0.788 &  0.020 &  0.319 &  0.010 &  0.709 &  0.043 &  0.384 &  0.016 &  0.033 &  0.037 \\
\hline
\end{tabular}}
\end{center}

\begin{center}
\scalebox{0.83}[0.83]{
    \begin{tabular}{l|l|ll|ll|ll|ll|ll|ll|ll|ll}
\hline
 Dataset &     \multicolumn{17}{c}{cod-rna}\\
 \hline
 Attack &     &   \multicolumn{4}{c|}{Attack $\varepsilon=0$}    &  \multicolumn{4}{c|}{Attack $\varepsilon=0.001$} &   \multicolumn{4}{c|}{Attack $\varepsilon=0.01$} &  \multicolumn{4}{c}{Attack $\varepsilon=0.1$} \\
 \multirow{2}{*}{Method} &  \multirow{2}{*}{Cost} &    Err &        &    Rej &        &    Err &        &    Rej &        &    Err &        &   Rej &       &    Err &      &   Rej &    \\
  &  & Mean &    Std &   Mean &    Std &   Mean &    Std &   Mean &    Std &   Mean &    Std &  Mean &   Std &   Mean &  Std &  Mean &  Std    \\
\hline
SVM & - &  0.114 &  0.141 &  - & - &  0.147 &  0.213 &  - & - &  0.243 &  0.321 &  - & - &  0.288 &  0.353 &  - & - \\
AT & - &  0.112 &  0.137 &  - & - &  0.112 &  0.137 &  - & - &  0.112 &  0.137 &  - & - &  0.112 &  0.137 &  - & - \\
MH & 0.2 &  \textbf{0.078} &  0.096 &  0.391 &  0.479 &  \textbf{0.078} &  0.095 &  0.386 &  0.473 &  0.288 &  0.353 &  0.0 &  0.00 &  0.288 &  0.353 &  0.0 &  0.00 \\
      & 0.3 &  0.112 &  0.138 &  0.306 &  0.397 &  0.144 &  0.190 &  0.247 &  0.388 &  0.281 &  0.344 &  0.0 &  0.00 &  0.288 &  0.353 &  0.0 &  0.00 \\
      & 0.4 &  0.083 &  0.126 &  0.000 &  0.000 &  0.083 &  0.126 &  0.000 &  0.000 &  0.172 &  0.286 &  0.0 &  0.00 &  0.216 &  0.330 &  0.0 &  0.00 \\
ATRO & 0.2 &  0.080 &  0.098 &  0.400 &  0.490 &  0.080 &  0.098 &  0.400 &  0.490 &  \textbf{0.080} &  0.098 &  0.4 &  0.49 &  \textbf{0.080} &  0.098 &  0.4 &  0.49 \\
      & 0.3 &  0.120 &  0.147 &  0.400 &  0.490 &  0.120 &  0.147 &  0.400 &  0.490 &  0.120 &  0.147 &  0.4 &  0.49 &  0.120 &  0.147 &  0.4 &  0.49 \\
      & 0.4 &  0.084 &  0.128 &  0.000 &  0.000 &  0.084 &  0.128 &  0.000 &  0.000 &  0.084 &  0.128 &  0.0 &  0.00 &  0.084 &  0.128 &  0.0 &  0.00 \\
\bottomrule
\end{tabular}}
\end{center}

\begin{center}
\scalebox{0.83}[0.83]{
    \begin{tabular}{l|l|ll|ll|ll|ll|ll|ll|ll|ll}
\hline
 Dataset &     \multicolumn{17}{c}{skin}\\
 \hline
 Attack &     &   \multicolumn{4}{c|}{Attack $\varepsilon=0$}    &  \multicolumn{4}{c|}{Attack $\varepsilon=0.001$} &   \multicolumn{4}{c|}{Attack $\varepsilon=0.01$} &  \multicolumn{4}{c}{Attack $\varepsilon=0.1$} \\
 \multirow{2}{*}{Method} &  \multirow{2}{*}{Cost} &    Err &        &    Rej &        &    Err &        &    Rej &        &    Err &        &   Rej &       &    Err &      &   Rej &    \\
  &  & Mean &    Std &   Mean &    Std &   Mean &    Std &   Mean &    Std &   Mean &    Std &  Mean &   Std &   Mean &  Std &  Mean &  Std    \\
\hline
SVM & - &  0.007 &  0.003 &  - & - &  0.012 &  0.009 &  - & - &  0.287 &  0.323 &  - & - &  0.770 &  0.083 &  - & - \\
AT & - &  0.034 &  0.038 &  - &  - &  0.034 &  0.036 &  - &  - &  0.029 &  0.038 &  - &  - &  0.775 &  0.065 &  - &  - \\
MH & 0.2 &  0.007 &  0.004 &  0.000 &  0.000 &  0.010 &  0.005 &  0.0 &  0.0 &  0.534 &  0.350 &  0.0 &  0.0 &  0.774 &  0.094 &  0.0 &  0.0 \\
      & 0.3 &  0.008 &  0.003 &  0.000 &  0.000 &  0.007 &  0.006 &  0.0 &  0.0 &  0.081 &  0.207 &  0.0 &  0.0 &  \textbf{0.661} &  0.149 &  0.0 &  0.0 \\
      & 0.4 &  \textbf{0.006} &  0.004 &  0.001 &  0.002 &  0.009 &  0.006 &  0.0 &  0.0 &  0.149 &  0.287 &  0.0 &  0.0 &  0.747 &  0.086 &  0.0 &  0.0 \\
ATRO & 0.2 &  0.007 &  0.003 &  0.000 &  0.000 &  \textbf{0.006} &  0.004 &  0.0 &  0.0 &  0.005 &  0.005 &  0.0 &  0.0 &  0.814 &  0.018 &  0.0 &  0.0 \\
      & 0.3 &  0.007 &  0.004 &  0.000 &  0.000 &  0.006 &  0.004 &  0.0 &  0.0 &  \textbf{0.003} &  0.004 &  0.0 &  0.0 &  0.799 &  0.042 &  0.0 &  0.0 \\
      & 0.4 &  0.007 &  0.004 &  0.000 &  0.000 &  0.007 &  0.004 &  0.0 &  0.0 &  0.004 &  0.005 &  0.0 &  0.0 &  0.809 &  0.030 &  0.0 &  0.0 \\
\hline
\end{tabular}}
\end{center}

\end{table*}

\begin{table*}[!ht]
    \caption{Results with linear-in-parameter models with an adversarial training parameter $\varepsilon=0.1$. We show the error (Err) and rejection rate (Rej) with their average (Mean) and standard deviations (SD). The lowest Err results are in bold. }
    \label{tbl:res_linear_4}
    \begin{center}

    \scalebox{0.83}[0.83]{
    \begin{tabular}{l|l|ll|ll|ll|ll|ll|ll|ll|ll}
\hline
 Dataset &     \multicolumn{17}{c}{australian}\\
 \hline
 Attack &     &   \multicolumn{4}{c|}{Attack $\varepsilon=0$}    &  \multicolumn{4}{c|}{Attack $\varepsilon=0.001$} &   \multicolumn{4}{c|}{Attack $\varepsilon=0.01$} &  \multicolumn{4}{c}{Attack $\varepsilon=0.1$} \\
 \multirow{2}{*}{Method} &  \multirow{2}{*}{Cost} &    Err &        &    Rej &        &    Err &        &    Rej &        &    Err &        &   Rej &       &    Err &      &   Rej &    \\
  &  & Mean &    Std &   Mean &    Std &   Mean &    Std &   Mean &    Std &   Mean &    Std &  Mean &   Std &   Mean &  Std &  Mean &  Std    \\
\hline
SVM & - &  0.199 &  0.008 &  - & - &  0.315 &  0.012 &  - & - &  0.529 &  0.025 &  - & - &  0.553 &  0.000 &  - & - \\
AT & - &  0.202 &  0.026 &  - & - &  0.201 &  0.025 &  - & - &  0.226 &  0.089 &  - & - &  0.325 &  0.150 &  - & - \\
MH & 0.2 &  \textbf{0.149} &  0.010 &  0.188 &  0.013 &  0.201 &  0.012 &  0.060 &  0.013 &  0.488 &  0.025 &  0.000 &  0.000 &  0.553 &  0.000 &  0.000 &  0.000 \\
      & 0.3 &  0.175 &  0.008 &  0.093 &  0.029 &  0.249 &  0.033 &  0.008 &  0.011 &  0.497 &  0.029 &  0.000 &  0.000 &  0.553 &  0.000 &  0.000 &  0.000 \\
      & 0.4 &  0.191 &  0.010 &  0.050 &  0.040 &  0.251 &  0.021 &  0.001 &  0.002 &  0.503 &  0.028 &  0.000 &  0.000 &  0.553 &  0.000 &  0.000 &  0.000 \\
ATRO & 0.2 &  0.200 &  0.000 &  1.000 &  0.000 &  \textbf{0.200} &  0.000 &  1.000 &  0.000 &  \textbf{0.200} &  0.000 &  1.000 &  0.000 &  \textbf{0.200} &  0.000 &  0.996 &  0.010 \\
      & 0.3 &  0.300 &  0.000 &  1.000 &  0.000 &  0.300 &  0.000 &  1.000 &  0.000 &  0.300 &  0.000 &  1.000 &  0.000 &  0.299 &  0.002 &  0.996 &  0.012 \\
      & 0.4 &  0.396 &  0.013 &  0.979 &  0.062 &  0.396 &  0.013 &  0.979 &  0.062 &  0.395 &  0.016 &  0.965 &  0.106 &  0.355 &  0.068 &  0.719 &  0.393 \\
\bottomrule
\end{tabular}}
\end{center}

\begin{center}
\scalebox{0.83}[0.83]{
    \begin{tabular}{l|l|ll|ll|ll|ll|ll|ll|ll|ll}
\hline
 Dataset &     \multicolumn{17}{c}{diabetes}\\
 \hline
 Attack &     &   \multicolumn{4}{c|}{Attack $\varepsilon=0$}    &  \multicolumn{4}{c|}{Attack $\varepsilon=0.001$} &   \multicolumn{4}{c|}{Attack $\varepsilon=0.01$} &  \multicolumn{4}{c}{Attack $\varepsilon=0.1$} \\
 \multirow{2}{*}{Method} &  \multirow{2}{*}{Cost} &    Err &        &    Rej &        &    Err &        &    Rej &        &    Err &        &   Rej &       &    Err &      &   Rej &    \\
  &  & Mean &    Std &   Mean &    Std &   Mean &    Std &   Mean &    Std &   Mean &    Std &  Mean &   Std &   Mean &  Std &  Mean &  Std    \\
\hline
SVM & - &  0.292 &  0.010 &  - & - &  0.369 &  0.017 &  - & - &  0.400 &  0.000 &  - & - &  0.400 &  0.000 &  - & - \\
AT & - &  0.400 &  0.000 &  - & - &  0.400 &  0.000 &  - & - &  0.400 &  0.000 &  - & - &  0.400 &  0.000 &  - & - \\
MH & 0.2 &  \textbf{0.171} &  0.002 &  0.749 &  0.012 &  0.297 &  0.080 &  0.177 &  0.254 &  0.400 &  0.000 &  0.000 &  0.000 &  0.400 &  0.000 &  0.000 &  0.000 \\
      & 0.3 &  0.227 &  0.011 &  0.472 &  0.072 &  0.312 &  0.033 &  0.001 &  0.002 &  0.400 &  0.000 &  0.000 &  0.000 &  0.400 &  0.000 &  0.000 &  0.000 \\
      & 0.4 &  0.271 &  0.016 &  0.130 &  0.064 &  0.331 &  0.020 &  0.000 &  0.000 &  0.400 &  0.000 &  0.000 &  0.000 &  0.400 &  0.000 &  0.000 &  0.000 \\
ATRO & 0.2 &  0.200 &  0.000 &  1.000 &  0.000 &  \textbf{0.200} &  0.000 &  1.000 &  0.000 &  \textbf{0.200} &  0.000 &  1.000 &  0.000 &  \textbf{0.236} &  0.072 &  0.843 &  0.329 \\
      & 0.3 &  0.310 &  0.030 &  0.899 &  0.300 &  0.310 &  0.030 &  0.899 &  0.300 &  0.310 &  0.030 &  0.899 &  0.300 &  0.320 &  0.040 &  0.798 &  0.399 \\
      & 0.4 &  0.347 &  0.030 &  0.353 &  0.320 &  0.354 &  0.034 &  0.311 &  0.305 &  0.393 &  0.026 &  0.123 &  0.293 &  0.400 &  0.000 &  0.100 &  0.300 \\
\hline
\end{tabular}}
\end{center}

\begin{center}
\scalebox{0.83}[0.83]{
    \begin{tabular}{l|l|ll|ll|ll|ll|ll|ll|ll|ll}
\hline
 Dataset &     \multicolumn{17}{c}{cod-rna}\\
 \hline
 Attack &     &   \multicolumn{4}{c|}{Attack $\varepsilon=0$}    &  \multicolumn{4}{c|}{Attack $\varepsilon=0.001$} &   \multicolumn{4}{c|}{Attack $\varepsilon=0.01$} &  \multicolumn{4}{c}{Attack $\varepsilon=0.1$} \\
 \multirow{2}{*}{Method} &  \multirow{2}{*}{Cost} &    Err &        &    Rej &        &    Err &        &    Rej &        &    Err &        &   Rej &       &    Err &      &   Rej &    \\
  &  & Mean &    Std &   Mean &    Std &   Mean &    Std &   Mean &    Std &   Mean &    Std &  Mean &   Std &   Mean &  Std &  Mean &  Std    \\
\hline
SVM & - &  0.081 &  0.123 &  - & - &  0.121 &  0.214 &  - & - &  0.215 &  0.328 &  - & - &  0.216 &  0.330 &  - & - \\
AT & - &  0.084 &  0.128 &  - &  - &  0.084 &  0.128 &  - &  - &  0.084 &  0.128 &  - &  - &  0.128 &  0.226 &  - &  - \\
MH & 0.2 &  0.059 &  0.090 &  0.296 &  0.452 &  0.059 &  0.090 &  0.293 &  0.448 &  0.215 &  0.328 &  0.0 &  0.000 &  0.216 &  0.330 &  0.0 &  0.000 \\
      & 0.3 &  0.085 &  0.131 &  0.078 &  0.209 &  0.120 &  0.204 &  0.069 &  0.206 &  0.215 &  0.328 &  0.0 &  0.000 &  0.216 &  0.330 &  0.0 &  0.000 \\
      & 0.4 &  0.084 &  0.128 &  0.000 &  0.000 &  0.084 &  0.128 &  0.000 &  0.000 &  0.215 &  0.329 &  0.0 &  0.000 &  0.216 &  0.330 &  0.0 &  0.000 \\
ATRO & 0.2 &  \textbf{0.060} &  0.092 &  0.300 &  0.458 &  \textbf{0.060} &  0.092 &  0.300 &  0.458 &  \textbf{0.060} &  0.092 &  0.3 &  0.458 &  \textbf{0.060} &  0.092 &  0.3 &  0.458 \\
      & 0.3 &  0.090 &  0.137 &  0.300 &  0.458 &  0.090 &  0.137 &  0.300 &  0.458 &  0.090 &  0.137 &  0.3 &  0.458 &  0.090 &  0.137 &  0.3 &  0.458 \\
      & 0.4 &  0.108 &  0.168 &  0.200 &  0.400 &  0.108 &  0.168 &  0.200 &  0.400 &  0.108 &  0.168 &  0.2 &  0.400 &  0.108 &  0.168 &  0.2 &  0.400 \\
\bottomrule
\end{tabular}}
\end{center}

\begin{center}
\scalebox{0.83}[0.83]{
    \begin{tabular}{l|l|ll|ll|ll|ll|ll|ll|ll|ll}
\hline
 Dataset &     \multicolumn{17}{c}{skin}\\
 \hline
 Attack &     &   \multicolumn{4}{c|}{Attack $\varepsilon=0$}    &  \multicolumn{4}{c|}{Attack $\varepsilon=0.001$} &   \multicolumn{4}{c|}{Attack $\varepsilon=0.01$} &  \multicolumn{4}{c}{Attack $\varepsilon=0.1$} \\
 \multirow{2}{*}{Method} &  \multirow{2}{*}{Cost} &    Err &        &    Rej &        &    Err &        &    Rej &        &    Err &        &   Rej &       &    Err &      &   Rej &    \\
  &  & Mean &    Std &   Mean &    Std &   Mean &    Std &   Mean &    Std &   Mean &    Std &  Mean &   Std &   Mean &  Std &  Mean &  Std    \\
\hline
SVM & - &  0.008 &  0.004 &  - & - &  0.015 &  0.015 &  - & - &  0.327 &  0.395 &  - & - &  0.761 &  0.111 &  - & - \\
AT & - &  0.237 &  0.209 &  - & - &  0.227 &  0.213 &  - & - &  0.313 &  0.208 &  - & - &  0.583 &  0.195 &  - & - \\
MH & 0.2 &  \textbf{0.008} &  0.004 &  0.001 &  0.002 &  \textbf{0.011} &  0.008 &  0.000 &  0.000 &  0.278 &  0.332 &  0.000 &  0.000 &  0.693 &  0.130 &  0.000 &  0.000 \\
      & 0.3 &  0.012 &  0.003 &  0.001 &  0.002 &  0.014 &  0.005 &  0.001 &  0.002 &  0.276 &  0.316 &  0.000 &  0.000 &  0.738 &  0.103 &  0.000 &  0.000 \\
      & 0.4 &  0.009 &  0.006 &  0.000 &  0.000 &  0.012 &  0.006 &  0.000 &  0.000 &  0.703 &  0.233 &  0.000 &  0.000 &  0.795 &  0.074 &  0.000 &  0.000 \\
ATRO & 0.2 &  0.145 &  0.030 &  0.725 &  0.150 &  0.142 &  0.031 &  0.712 &  0.153 &  \textbf{0.117} &  0.037 &  0.583 &  0.185 &  0.234 &  0.077 &  0.193 &  0.274 \\
      & 0.3 &  0.161 &  0.076 &  0.365 &  0.357 &  0.162 &  0.076 &  0.360 &  0.355 &  0.150 &  0.068 &  0.317 &  0.320 &  \textbf{0.141} &  0.059 &  0.097 &  0.035 \\
      & 0.4 &  0.224 &  0.073 &  0.279 &  0.290 &  0.219 &  0.070 &  0.269 &  0.279 &  0.200 &  0.054 &  0.212 &  0.233 &  0.147 &  0.028 &  0.055 &  0.047 \\
\hline
\end{tabular}}
\end{center}

\end{table*}

\end{document}